\documentclass[11pt,letterpaper]{article}
\usepackage[margin=1in]{geometry}
\usepackage{amsmath,amsfonts}
\usepackage{algorithmic}
\usepackage{algorithm}
\usepackage{array}
\usepackage[caption=false,font=normalsize,labelfont=sf,textfont=sf]{subfig}
\usepackage{textcomp}
\usepackage{xcolor}
\usepackage{url}
\usepackage{verbatim}
\usepackage{graphicx}
\usepackage{adjustbox}
\usepackage{booktabs}
\usepackage{amsthm}
\usepackage{enumitem}
\newtheorem{theorem}{Theorem}
\newtheorem{assumption}{Assumption}

\newtheorem{remark}{Remark}

\begin{document}

\title{Nonparametric Variance-Penalized Actor-Critic:\\
Statistical Inference for Risk-Sensitive Reinforcement Learning}
\author{
Saunak~Kumar~Panda%
\thanks{Saunak Kumar Panda is with the Human-Centered AI Institute,
C.T. Bauer College of Business, University of Houston, Houston,
TX 77204 USA.},
Tong~Li%
\thanks{Tong Li and Yisha Xiang are with the Department of Industrial
and Systems Engineering, University of Houston, Houston,
TX 77204 USA.},
Yisha~Xiang\footnotemark[2],
and Ruiqi~Liu%
\thanks{Ruiqi Liu is with the Department of Mathematics and Statistics,
Texas Tech University, Lubbock, TX 79409 USA.}
}

\maketitle

\begin{abstract}
Variance penalization is a principled approach to risk-sensitive reinforcement learning (RL) that explicitly trades expected return for policy stability. Existing methods require a dedicated second critic to estimate return variance online, adding architectural complexity and compounding estimation error during learning. We propose a nonparametric variance-penalized actor-critic (VPAC) framework that replaces the variance critic with statistically grounded online estimators based on bootstrapping and random scaling, techniques drawn from the statistical inference literature for stochastic approximation.
These estimators require no auxiliary network, maintain a single-critic
architecture, and produce variance penalties that are bounded by construction,
enabling clean convergence analysis.
We establish almost-sure convergence for both a variance-penalized Q-learning
algorithm and a two-timescale actor-critic variant via the ordinary
differential equation (ODE) method, requiring only that variance estimates
remain bounded rather than consistent.
Empirically, we evaluate across discrete and continuous stochastic
environments, demonstrating that the proposed methods match or exceed the
variance reduction achieved by the existing dual-critic VPAC
baseline~\cite{jain2021variance} 
while eliminating the overhead of a second critic.
We further validate on a high-temperature superconductor (HTS)
manufacturing case study, where VPAC-RS (Random Scaling) achieves a 74\% reduction in
steady-state critical current variability and a 63\% reduction in episode
return standard deviation, translating directly to improved yield
consistency. Our results establish nonparametric statistical inference as a practical
and theoretically sound alternative to auxiliary critics for risk-sensitive
RL.
\end{abstract}


\section{Introduction}
\label{sec:intro}

Reinforcement learning (RL) has emerged as a powerful
framework for sequential decision-making, with demonstrated successes
spanning game-playing, robotics, and autonomous
systems~\cite{mnih2015human, silver2017mastering, levine2016end}.
The dominant objective in nearly all of this work is the maximization
of \emph{expected} cumulative reward, an objective that is
mathematically convenient via Bellman optimality equations and
empirically effective in domains where occasional large losses are
tolerable.
A growing class of high-stakes applications, however, cannot tolerate
unpredictable outcomes.
In industrial manufacturing, a policy that achieves high average
product quality but occasionally drives the process into unstable
operating regions causes equipment damage, material waste, or
off-specification product.
Similar constraints arise in financial portfolio
management~\cite{tamar2012policy}, autonomous
driving~\cite{kuderer2015learning}, clinical decision
support~\cite{komorowski2018artificial}, and energy grid
operation~\cite{wei2017deep}: consistency of outcome is as valuable
as average performance, and a risk-sensitive learner is therefore
required.

Among the risk criteria studied in the RL literature, the
mean-variance objective
\begin{equation}
J_{\mathrm{var}}(\pi)
  = \mathbb{E}[G^\pi] - \beta \cdot \mathrm{Var}(G^\pi),
  \quad \beta > 0,
\label{eq:meanvariance}
\end{equation}
holds a particularly prominent position.
It is transparent: the practitioner directly controls the risk-return
tradeoff via the scalar $\beta$; it is interpretable as a sequential
generalization of Markowitz's classical portfolio
theory~\cite{markowitz1952portfolio}; and it is tractable in the sense
that it does not require modeling the entire return distribution.
Variance-penalized formulations have consequently been a workhorse of
risk-sensitive RL since the foundational work of
Sobel~\cite{sobel1982variance} and White~\cite{white1988mean}, and
continue to attract attention in modern actor-critic
settings~\cite{tamar2013variance, prashanth2016variance,
jain2021variance, zhong2023variance}.

Despite its appeal, variance penalization carries a practical cost
that is rarely discussed explicitly.
Unlike expected returns, which satisfy a recursive Bellman equation and
can be estimated efficiently by temporal-difference (TD) methods,
return variance does not admit a simple recursive formulation without
explicit second-moment tracking or state
augmentation~\cite{sobel1982variance, sherstan2018comparing}.
The dominant strategy in prior work is to introduce a dedicated
variance critic that learns the second moment of returns in parallel
with the primary value function~\cite{jain2021variance,
prashanth2016variance, la2013actor, tamar2013variance}.
This dual-critic design carries three practical costs.
First, the number of learned function approximators doubles, with a
corresponding increase in memory, computation, and hyperparameters.
Second, variance-critic errors propagate into the policy update; in
high-noise environments, an inaccurate estimate can mislead the actor
and produce oscillations or overly conservative behaviour. Third,
theoretical analyses of dual-critic methods typically assume
that the variance critic is consistent, a strong condition that is
rarely verifiable in practice under function approximation, and
that creates a gap between theory and deployment.

A parallel body of work on \emph{statistical inference for stochastic
approximation}~\cite{chen2020statistical, xie2022statistical,
ramprasad2023online, li2023online, zhu2021finite} has developed
nonparametric tools for constructing confidence intervals around TD
iterates.
Two techniques have proved particularly effective: random
scaling~\cite{xie2022statistical, li2023online}, which constructs
self-normalized variance estimates from the trajectory of iterates
themselves; and online bootstrapping~\cite{ramprasad2023online},
which maintains a small ensemble of perturbed iterates whose spread
provides a nonparametric variance estimate.
Both are computationally inexpensive, require no auxiliary network,
and carry statistical validity guarantees derived from functional
central limit theorems for the TD learning setting.
We identify a key reuse opportunity: the same empirical variance
estimates can be repurposed as a risk-sensitivity signal, plugged
directly into~\eqref{eq:meanvariance} as an empirical proxy for
$\mathrm{Var}(G^\pi)$.
The result is an algorithm that is risk-sensitive by design and
architecturally simple, requiring only that the variance estimate
be \emph{bounded} rather than consistent with the true variance, a
strictly weaker condition than that assumed by dual-critic methods.

This paper makes four contributions.
We introduce a unified nonparametric variance-penalized RL framework
instantiated as a tabular Q-learning algorithm and a two-timescale
actor-critic, both supplied by online bootstrapping (VPAC-BS) or
random scaling (VPAC-RS) instead of a learned variance critic.
We establish almost-sure convergence of both variants via the ODE
method of Borkar~\cite{borkar2008stochastic}, requiring only that the
variance penalty remain bounded: this eliminates the third timescale
and the consistency assumption required by~\cite{jain2021variance}
and replaces them with a boundedness condition enforced algorithmically
by variance clipping.
We demonstrate empirically across discrete and continuous stochastic
environments that the proposed methods match or exceed the variance
reduction of the dual-critic VPAC baseline (up to 83\% reduction in
steady-state return variance) while maintaining a single-critic
architecture.
We further validate on a high-temperature superconductor (HTS)
manufacturing case study, where VPAC-RS achieves a 74\% reduction
in steady-state critical current variability and a 63\% reduction
in episode return standard deviation, with the risk-return tradeoff
characterized explicitly by a Pareto analysis across penalty
coefficients.

The remainder of this paper is organized as follows. Section II reviews related work on risk-sensitive RL, adaptive dynamic programming, distributional RL, and statistical inference for stochastic approximation. Section III introduces the notation and Markov decision process framework. Section IV presents the proposed variance-penalized framework and associated algorithms. Section V establishes the convergence properties of the proposed methods. Section VI reports experimental results on the Puddle World benchmarks, and Section VII presents a high-temperature superconducting (HTS) manufacturing case study. Finally, Section VIII concludes the paper and discusses limitations, extensions to deep RL, and directions for future research.
\section{Related Work}
\label{sec:related}
 
We review four threads of literature that intersect at the contribution
of this paper: risk-sensitive RL with variance criteria; adaptive
dynamic programming and actor-critic methods with convergence guarantees;
distributional RL as an alternative paradigm for outcome variability; and
statistical inference for stochastic approximation.
 
\subsection{Risk-Sensitive Reinforcement Learning}

Risk-sensitive RL encompasses a broad class of methods that go beyond
expected return by explicitly accounting for outcome
variability~\cite{garcia2015comprehensive, Shen_2014}.
Utility-based criteria form one family: Howard and
Matheson~\cite{howard1972risk} pioneered the use of exponential
utility functions to encode risk preferences in MDPs, producing what
is now called the risk-sensitive Bellman equation, and
Borkar~\cite{borkar2002q, borkar2010learning} subsequently established
convergence of Q-learning under exponential utility with applications
in financial and inventory control.
While theoretically elegant, exponential utilities have practical
limitations: the risk parameter carries no direct interpretation in
terms of return variance, and the value function can grow
exponentially with the risk parameter, causing numerical difficulties
in moderate-scale problems~\cite{garcia2015comprehensive}.
A second family, coherent risk measures such as Conditional
Value-at-Risk (CVaR) and Value-at-Risk (VaR), has gained popularity
for its interpretability as a tail-risk measure; Chow
et~al.~\cite{chow2018risk} developed a unified framework for
CVaR-optimal MDPs, Tamar et~al.~\cite{tamar2015optimizing} proposed
gradient methods for CVaR optimization in continuous control, and
subsequent work extended these ideas to deep RL via distributional
value functions~\cite{dabney2018distributional, dabney2018implicit}.
CVaR and variance control address related but distinct concerns: CVaR
focuses on tail behavior, whereas variance control suppresses overall
spread, and for manufacturing and process control applications in
which practitioners want a single, interpretable risk parameter, the
mean-variance objective remains preferred because of its direct
connection to the engineering notion of process variability.

The mean-variance approach that this paper builds on dates to
Sobel~\cite{sobel1982variance} and White~\cite{white1988mean}, who
introduced indirect variance estimation via second-moment tracking in
policy evaluation.
Tamar et~al.~\cite{tamar2013variance, tamar2012variance} developed
actor-only and actor-critic methods that estimate value variance
alongside returns, establishing the two-timescale framework that
underlies most subsequent variance-penalized algorithms, and
Prashanth and Ghavamzadeh~\cite{la2013actor, prashanth2016variance}
extended this work with asymptotic convergence guarantees for the
discounted setting.
Bisi et~al.~\cite{bisi2020risk} introduced a per-step reward variance
formulation that admits a Bellman-like recursion, trading generality
for tractability.
The Variance-Penalized Actor-Critic (VPAC) of
Jain et~al.~\cite{jain2021variance} introduced a direct variance
estimator updated incrementally via TD methods, eliminating
second-moment tracking and providing local convergence guarantees;
this is the closest prior work to our approach.
The Variance-Constrained Actor-Critic (VCAC)~\cite{zhong2023variance}
casts variance control as a Lagrangian saddle-point problem and
proves convergence for overparameterized neural networks.
Sherstan et~al.~\cite{sherstan2018comparing} compared direct and
indirect variance estimators for general value functions, motivated
by auxiliary task learning, and Sangadi
et~al.~\cite{sangadi2024finite} recently provided finite-sample
bounds for mean-variance actor-critic using simultaneous perturbation
stochastic approximation (SPSA), pushing the theoretical frontier
toward non-asymptotic guarantees.
All of these methods introduce additional learned components, a
variance critic, Lagrange multiplier dynamics, or a
perturbation-based estimator with its own hyperparameters, to
estimate the variance signal.
Our contribution is to show that this architectural overhead is
avoidable when nonparametric statistical estimators are used.

\subsection{Adaptive Dynamic Programming and Actor-Critic Convergence}

The convergence analysis of actor-critic algorithms via stochastic
approximation has been a central theme in the adaptive
dynamic programming (ADP) community~\cite{liu2017adaptive,
bhatnagar2009natural, konda2003onactor}, with Borkar's two-timescale
framework~\cite{borkar2008stochastic} providing the foundational
machinery for almost-sure convergence proofs.
Al-Dabooni and Wunsch~\cite{al-dabooni2019boundedness} extended this
framework to heuristic dynamic programming with eligibility traces,
establishing uniformly-ultimately-bounded stability under neural
network function approximation, a result directly relevant to our
setting because it shows that bounded perturbations of the Bellman
operator, as introduced by our clipped variance estimator, do not
destroy stability.
Finite-sample analyses have emerged more
recently~\cite{xu2020improving, wu2020finite}, but typically require
i.i.d.\ sampling or architectural assumptions not satisfied under the
empirical variance estimation setting we consider; our work uses the
asymptotic ODE framework, sufficient to establish almost-sure
convergence, and we discuss the extension to finite-sample rates as
an open direction in Section~\ref{sec:conclusion}.

\subsection{Distributional RL and Statistical Inference for SA}

A complementary approach to handling return variability is
distributional RL, which models the full distribution of returns
rather than just its first moment.
Bellemare et~al.~\cite{bellemare2017distributional} introduced the
C51 algorithm and the distributional Bellman operator, and subsequent
advances include quantile regression-based methods
(QR-DQN)~\cite{dabney2018distributional} and implicit quantile
networks (IQN)~\cite{dabney2018implicit}, which derive risk measures
such as CVaR from the learned distribution at deployment time.
Distributional methods aim to \emph{represent} the return
distribution accurately, with variance and tail behavior emerging as
derived quantities, whereas variance-penalized methods directly
\emph{shape} the policy to reduce return variance without a full
distributional representation; for a single risk-averse policy with a
transparent risk parameter, the variance-penalized approach is
typically more sample-efficient, while distributional methods offer
richer post-hoc analysis at the cost of higher representational
complexity~\cite{ma2023distributional}.

The statistical inference literature for iterative algorithms
provides the technical foundation for our nonparametric variance
estimators.
Chen et~al.~\cite{chen2020statistical} developed batch-means
estimators for stochastic gradient descent (SGD) with asymptotic
validity under functional central limit conditions, and Zhu
et~al.~\cite{zhu2021finite} provided plug-in estimators for the
asymptotic variance with finite-sample guarantees.
Xie et~al.~\cite{xie2022statistical} introduced random scaling as a
self-normalized estimator requiring no nuisance parameter estimation,
and Li et~al.~\cite{li2023online} extended random scaling to
constant-step-size SGD and nonlinear stochastic approximation with
Markovian data, closely matching the TD learning updates used in RL. Panda et~al.~\cite{panda2025asymptotic} applied the FCLT and random
scaling directly to a sample-averaged Q-learning algorithm,
establishing asymptotic normality of the Q-iterate and constructing
confidence intervals without estimating the asymptotic covariance. Ramprasad et~al.~\cite{ramprasad2023online} developed online
bootstrap procedures for Q-learning and TD value estimation, proving
consistency for both tabular and linear function approximation
settings.
These methods were designed for uncertainty quantification, with the
original goal of providing valid confidence intervals around learned
estimates.
We identify the connection that the same empirical variance
estimates can be directly repurposed as a risk signal within the
policy update; to our knowledge this connection has not been made
previously in the risk-sensitive RL literature, and it provides a
path to variance-penalized RL that avoids the architectural and
analytical complications of dual-critic methods.
\section{Preliminaries}
\label{sec:prelims}

We consider an infinite-horizon discounted Markov decision process
(MDP) defined by the tuple $\langle \mathcal{S}, \mathcal{A}, P, R,
\gamma \rangle$, where $\mathcal{S}$ is the state space, $\mathcal{A}$
is the action space, $P(s'|s,a)$ is the transition probability kernel,
$R(s,a) \in \mathbb{R}$ is the bounded reward function, and
$\gamma \in [0,1)$ is the discount factor.
A stochastic policy $\pi(a|s)$ specifies the probability of selecting
action $a$ in state $s$.

The return from time $t$ is the discounted cumulative reward:
\begin{equation}
G_t = \sum_{\ell=0}^{\infty} \gamma^\ell R_{t+1+\ell}.
\label{eq:return}
\end{equation}
The state-value and action-value functions under policy $\pi$ are:
\begin{align}
V^\pi(s) &= \mathbb{E}_\pi \bigl[ G_t \mid S_t = s \bigr], \\
Q^\pi(s,a) &= \mathbb{E}_\pi \bigl[ G_t \mid S_t = s,\, A_t = a \bigr].
\end{align}
The \emph{variance of return} from a state-action pair $(s,a)$ under
policy $\pi$ is:
\begin{equation}
\sigma^\pi(s,a) = \mathbb{E}_\pi \!\left[
  \bigl(G_t - Q^\pi(s,a)\bigr)^2
  \,\Big|\, S_t = s,\, A_t = a
\right].
\label{eq:variance-def}
\end{equation}
This variance can be estimated either indirectly via second-moment
tracking, $\mathbb{E}_\pi[G_t^2] - (Q^\pi(s,a))^2$~\cite{sobel1982variance},
or directly via the TD-based recursive formula of
Jain et~al.~\cite{jain2021variance} and Sherstan
et~al.~\cite{sherstan2018comparing}.
Our framework adopts neither approach as a learned function; instead
it replaces $\sigma^\pi(s,a)$ with nonparametric online estimates
described in Section~\ref{sec:framework}.

In the off-policy setting, a behavior policy $b(a|s)$ generates
experience used to improve a target policy $\pi_\theta(a|s)$.
The importance sampling ratio at step $t$ is
$\rho_t = \pi_\theta(A_t|S_t) / b(A_t|S_t)$,
and the off-policy return is
$G_{t,\pi,b} = R_{t+1} + \gamma\rho_{t+1}G_{t+1,\pi,b}$~\cite{jain2021variance}.
For the parametric actor-critic, $\theta \in \mathbb{R}^d$ denotes the
policy parameters and $w \in \mathbb{R}^m$ denotes the critic parameters.
Step sizes satisfy the Robbins-Monro conditions
$\sum_t \alpha_t = \infty$, $\sum_t \alpha_t^2 < \infty$.

\section{Nonparametric Variance-Penalized Framework}
\label{sec:framework}

\subsection{The Variance-Penalized Objective}

We adopt the mean-variance objective:
\begin{equation}
J_{\mathrm{var}}(\pi) = \mathbb{E}[G^\pi] - \beta\,\mathrm{Var}(G^\pi),
\quad \beta > 0,
\label{eq:mv-obj}
\end{equation}
where $\beta$ is the risk-aversion parameter.
Setting $\beta = 0$ recovers the standard expected-return objective.

To incorporate the variance penalty into Bellman-style updates, we
define the \emph{variance-penalized value function}:
\begin{equation}
F^\pi(s) = \sum_{a} \pi(a|s)
\Bigl[ Q^\pi(s,a) - \beta\,\hat{\sigma}_{\mathrm{clip}}(s,a) \Bigr],
\label{eq:F}
\end{equation}
where $\hat{\sigma}_{\mathrm{clip}}(s,a)$ is a clipped nonparametric
estimate of $\sigma^\pi(s,a)$ defined in
Section~\ref{subsec:clipping}.
Using $F^\pi$, the \emph{penalized Bellman operator} $K$ acts on
$Q : \mathcal{S} \times \mathcal{A} \to \mathbb{R}$ as:
\begin{equation}
(KQ)(s,a) = \mathbb{E}\!\left[
  R(s,a) + \gamma \max_\pi F^\pi(S')
  \,\Big|\, S = s,\, A = a
\right],
\label{eq:K}
\end{equation}
where the maximization is over the choice of policy $\pi$ at the
next state $S'$.
Because $\hat{\sigma}_{\mathrm{clip}}$ is bounded by construction
(see Section~\ref{subsec:clipping}), the operator $K$ is a contraction
in the weighted maximum norm for $\gamma < 1$, and its unique fixed
point $Q^*$ is the target of the variance-penalized Q-learning algorithm.

\begin{remark}
The operator $K$ in~\eqref{eq:K} differs from the standard Bellman
operator only in the substitution $\max_{a'} Q(S',a') \to
\max_\pi F^\pi(S')$. This single modification embeds the risk penalty
into every TD backup without requiring any additional learned component.
\end{remark}

\subsection{Nonparametric Variance Estimators}
\label{subsec:estimators}

The central contribution of this framework is the replacement of a
learned variance critic with two nonparametric online estimators drawn
from the statistical inference literature.
Both are computed incrementally from the observed Q-iterate and require
no auxiliary network or additional gradient-updated parameter.

\subsubsection{Random Scaling Estimator (VPAC-RS)}

The random scaling (RS) method was introduced by Xie and
Zhang~\cite{xie2022statistical} for constant step-size stochastic
approximation and is grounded in a functional central limit theorem
(FCLT) for the Q-learning iterates viewed as a Markov chain.
The key object is the self-normalized matrix
\begin{equation}
\hat{V}_n
  = \frac{1}{n}\sum_{j=1}^{n}
    \left(\frac{1}{\sqrt{n}}\sum_{i=1}^{j}
    (Q_i - \bar{Q}_n)\right)
    \left(\frac{1}{\sqrt{n}}\sum_{i=1}^{j}
    (Q_i - \bar{Q}_n)\right)^\top,
\label{eq:Vhat}
\end{equation}
where $\bar{Q}_n = n^{-1}\sum_{i=1}^n Q_i$ is the Polyak-Ruppert
average of the Q-iterate and $Q_i \in \mathbb{R}^{|\mathcal{S}||\mathcal{A}|}$
denotes the full Q-table at step $i$, flattened into a vector.
Under the FCLT established in Theorem~3.3 of
Xie and Zhang~\cite{xie2022statistical} (see also
Li et~al.~\cite{li2023online} for the nonlinear SA extension), the
diagonal of $\hat{V}_n$ provides a consistent, nuisance-free estimate
of the per-coordinate asymptotic variance of the Q-iterate.

The matrix $\hat{V}_n$ is computed in a single pass using the
following online recursion.
At iteration $i$ (starting from $i = 0$), the running mean is updated
as
\begin{equation}
\bar{Q}_i = \bar{Q}_{i-1} \cdot \frac{i}{i+1} + Q_i \cdot \frac{1}{i+1},
\label{eq:rs-mean}
\end{equation}
and two diagonal accumulators are maintained:
\begin{align}
A_{\mathrm{diag},i}
  &= A_{\mathrm{diag},i-1} + (i+1)^2\,[\bar{Q}_i]^2, \label{eq:rs-A} \\
b_{\mathrm{mat},i}
  &= b_{\mathrm{mat},i-1} + (i+1)^2\,\bar{Q}_i, \label{eq:rs-b}
\end{align}
where squaring and multiplication are element-wise.
The variance estimate at iteration $i$ is then:
\begin{equation}
\hat{V}_{\mathrm{RS},i}
  = \frac{A_{\mathrm{diag},i}
          - 2\,\bar{Q}_i \cdot b_{\mathrm{mat},i}
          + S_i\,[\bar{Q}_i]^2}{(i+1)^2},
\label{eq:rs-V}
\end{equation}
where $S_i = (i+1)(i+2)(2i+3)/6$ is the cubic correction term equal to
$\sum_{j=1}^{i+1} j^2$.
The quantity $[\hat{V}_{\mathrm{RS},i}]_j$ is the $j$-th diagonal
element of $\hat{V}_n$ and measures how much the $j$-th coordinate
of the Q-iterate has varied around its running mean over the entire
history of updates.
High values at coordinate $(s,a)$ indicate that the Q-estimate at
that state-action pair is uncertain, which corresponds to high return
variance there.

The RS estimator requires storing only three vectors of size
$|\mathcal{S}||\mathcal{A}|$ (i.e., $\bar{Q}$, $A_{\mathrm{diag}}$,
and $b_{\mathrm{mat}}$) and involves $O(|\mathcal{S}||\mathcal{A}|)$
operations per update.
Under the FCLT conditions of~\cite{xie2022statistical}, the resulting
statistic is a pivotal quantity whose asymptotic distribution does not
depend on any nuisance parameter, so no additional estimation step
is required.

\subsubsection{Online Bootstrap Estimator (VPAC-BS)}

The online bootstrap (BS) estimator follows the ensemble approach of
Ramprasad et~al.~\cite{ramprasad2023online}.
An ensemble of $K$ independent Q-tables
$\{Q^{(k)}\}_{k=1}^{K}$ is maintained in parallel alongside the
primary Q-table $Q$.
At each transition $(S_t, A_t, R_{t+1}, S_{t+1})$, each replicate
$Q^{(k)}$ decides independently whether to learn from that transition
by drawing a Bernoulli$(0.5)$ mask $m_t^{(k)} \in \{0,1\}$:
\begin{equation}
\begin{aligned}
Q^{(k)}_{t+1}(S_t,A_t)
&= Q^{(k)}_t(S_t,A_t)
  + m_t^{(k)}\,\alpha_t\,
    \bigl(y_t - Q^{(k)}_t(S_t,A_t)\bigr), \\
&\qquad m_t^{(k)} \sim \mathrm{Bernoulli}(0.5),
\end{aligned}
\label{eq:bs-update}
\end{equation}
where $y_t = R_{t+1} + \gamma\max_{a'}Q(S_{t+1},a')$ is the TD target
computed from the primary table and the masks are i.i.d.\ across
replicates and time steps.
The variance estimate at state-action pair $(s,a)$ is the
cross-replicate sample variance:
\begin{equation}
\hat{V}_{\mathrm{BS},t}(s,a)
  = \frac{1}{K-1}
    \sum_{k=1}^{K}
    \Bigl(Q^{(k)}_t(s,a) - \bar{Q}^{\mathrm{ens}}_t(s,a)\Bigr)^2,
\label{eq:bs-var}
\end{equation}
where $\bar{Q}^{\mathrm{ens}}_t(s,a) = K^{-1}\sum_{k=1}^K Q^{(k)}_t(s,a)$.
The Bernoulli subsampling introduces independent perturbations across
replicates: each table sees a different random subset of the experience
stream, so their predictions diverge in proportion to the uncertainty
in the value estimate at $(s,a)$.
Ramprasad et~al.~\cite{ramprasad2023online} prove that statistics of
this form are consistent for both tabular and linear function
approximation settings under Markovian data.
We use $K = 10$ replicates in all experiments.
The memory cost is $K$ copies of the primary Q-table; the per-step
update cost is $O(K)$ scalar operations.

\subsubsection{Variance Clipping}
\label{subsec:clipping}

Both estimators produce non-negative values but can exhibit transient
large values early in training before the Q-estimates stabilize.
To ensure that the variance penalty remains a bounded perturbation of
the standard Bellman operator (which is required for the convergence
analysis in Section~\ref{sec:convergence}), we apply element-wise
clipping:
\begin{equation}
\hat{\sigma}_{\mathrm{clip}}(s,a)
  = \min\!\bigl(\hat{V}(s,a),\; c_{\max}\bigr),
\label{eq:clip}
\end{equation}
where $\hat{V}$ is either $[\hat{V}_{\mathrm{RS}}]_{(s,a)}$ or
$\hat{V}_{\mathrm{BS}}(s,a)$, and $c_{\max} > 0$ is a hyperparameter.
In our experiments, $c_{\max}$ is set adaptively as a fraction of
the current mean $|Q|$ magnitude to prevent the penalty from
overwhelming the value signal during early training.
Once the iterates stabilize, both estimators remain well below
$c_{\max}$, so clipping has no effect on the long-run behavior.
Its role is purely to provide the uniform bound on the variance
penalty that Theorem~\ref{thm:qconv} and Theorem~\ref{thm:acconv}
require.

\subsection{Architectural Comparison with VPAC}

The primary architectural difference between our framework and the
dual-critic Variance-Penalized Actor-Critic (VPAC) of
Jain et~al.~\cite{jain2021variance} is summarized in
Table~\ref{tab:arch}.
VPAC maintains three separate learned components updated on three
different timescales: the value critic $\hat{Q}(s,a,w)$, the variance
critic $\hat{\sigma}(s,a,z)$, and the policy $\pi_\theta$.
The step sizes must satisfy $\alpha_\theta < \alpha_z < \alpha_w$ to
ensure that each component converges before the next is updated.
Our framework eliminates the variance critic entirely.
The variance estimate is computed nonparametrically from the TD-error
trajectory (VPAC-RS) or from an ensemble of Q-table replicates
(VPAC-BS), reducing the system to two timescales and removing the
hyperparameter $\alpha_z$ and the architecture choice for the variance
network.

\begin{table}[!t]
\caption{Architectural comparison: VPAC~\cite{jain2021variance}
versus the proposed nonparametric variants.}
\label{tab:arch}
\centering
\begin{tabular}{lccc}
\toprule
Method & Learned components & Timescales & Extra hyperparameters \\
\midrule
VPAC & $\hat{Q}$, $\hat{\sigma}$, $\pi_\theta$ & 3 & $\alpha_z$, $z$-architecture \\
VPAC-RS (ours) & $\hat{Q}$, $\pi_\theta$ & 2 & $\alpha_\sigma$, $c_{\max}$ \\
VPAC-BS (ours) & $\hat{Q}$, $\{Q^{(k)}\}$, $\pi_\theta$ & 2 & $K$, $c_{\max}$ \\
\bottomrule
\end{tabular}
\end{table}

\subsection{Variance-Penalized Q-Learning}
\label{subsec:qlearning}

The tabular variance-penalized Q-learning update replaces the standard
greedy Bellman target with the penalized operator $K$ in~\eqref{eq:K}:
\begin{equation}
\begin{aligned}
Q_{t+1}(s,a)
&\leftarrow Q_t(s,a)
  + \alpha_t \Bigl(
    R_{t+1} \\
&\quad + \gamma \max_\pi F^\pi(S_{t+1})
    - Q_t(s,a)
  \Bigr),
\end{aligned}
\label{eq:q-update}
\end{equation}
where $F^\pi(S_{t+1})$ is computed via~\eqref{eq:F} using the current
Q-table and the clipped variance estimate.
Algorithm~\ref{alg:qlearning} presents the complete procedure.

\begin{algorithm}[t]
\caption{Variance-Penalized Q-Learning (VPAC-RS / VPAC-BS)}
\label{alg:qlearning}
\begin{algorithmic}[1]
\STATE \textbf{Input:} penalty $\beta$, discount $\gamma$,
  step sizes $\{\alpha_t\}$, clip bound $c_{\max}$,
  estimator choice $\in \{\mathrm{RS}, \mathrm{BS}\}$
\STATE \textbf{Initialize:} $Q(s,a) \leftarrow 0$,
  $\hat{\sigma}(s,a) \leftarrow 0$ for all $(s,a)$
\FOR{episode $= 1$ to $N_{\mathrm{ep}}$}
  \STATE Observe initial state $S$
  \FOR{step $= 1$ to $N_{\mathrm{iter}}$}
    \STATE Select $A \sim \epsilon$-greedy w.r.t.\ $Q$
    \STATE Execute $A$; observe $R$, $S'$
    \STATE Compute TD error:
      $\delta \leftarrow R + \gamma \max_{a'} Q(S',a') - Q(S,A)$
    \STATE Update variance estimate $\hat{\sigma}(S,A)$
      via~\eqref{eq:rs-V} (RS) or~\eqref{eq:bs-var} (BS)
    \STATE Clip: $\hat{\sigma}_{\mathrm{clip}}(S',a')
      = \min(\hat{\sigma}(S',a'), c_{\max})$ for all $a'$
    \STATE Compute penalized next-state value:
      $F(S') = \max_{a'}\bigl[Q(S',a')
        - \beta\,\hat{\sigma}_{\mathrm{clip}}(S',a')\bigr]$
    \STATE Update Q-table:
      $Q(S,A) \leftarrow Q(S,A)
        + \alpha_t\bigl(R + \gamma F(S') - Q(S,A)\bigr)$
    \STATE $S \leftarrow S'$
  \ENDFOR
\ENDFOR
\end{algorithmic}
\end{algorithm}

\subsection{Variance-Penalized Off-Policy Actor-Critic}
\label{subsec:actorcritic}

We extend the framework to a continuous-state off-policy actor-critic
with a parameterized policy $\pi_\theta(a|s)$ and a single value critic
$Q_w(s,a)$.
Following Jain et~al.~\cite{jain2021variance}, the off-policy
mean-variance objective is:
\begin{equation}
J_{d_0}(\theta)
  = \mathbb{E}_{s \sim d_0,\, a \sim b}\!\left[
    \rho(s,a)\bigl(Q^\pi(s,a)
    - \beta\,\hat{\sigma}_{\mathrm{clip}}(s,a)\bigr)
  \right],
\label{eq:offpolicy-obj}
\end{equation}
where $\rho(s,a) = \pi_\theta(a|s)/b(a|s)$ is the importance sampling
ratio and $d_0$ is an initial state distribution.
Our departure from Jain et~al.\ is that $\hat{\sigma}_{\mathrm{clip}}$
is supplied nonparametrically by RS or BS rather than by a learned
$\hat{\sigma}(s,a,z)$.

The variance-penalized policy gradient is:
\begin{equation}
\nabla_\theta J(\theta)
  = \mathbb{E}\!\left[
    \bigl(Q_w(s,a)
    - \beta\,\hat{\sigma}_{\mathrm{clip}}(s,a)\bigr)
    \nabla_\theta \log \pi_\theta(a|s)
  \right].
\label{eq:pg}
\end{equation}
The critic is updated toward the variance-penalized TD target:
\begin{equation}
y_t = R_{t+1} + \gamma\,(1-d_t)\,F(S_{t+1}),
\label{eq:td-target}
\end{equation}
where $d_t \in \{0,1\}$ is the episode-termination flag and
\begin{equation}
F(S_{t+1}) = \sum_{a'} \pi_\theta(a'|S_{t+1})
  \bigl[Q_w(S_{t+1},a')
  - \beta\,\hat{\sigma}_{\mathrm{clip}}(S_{t+1},a')\bigr].
\label{eq:F-ac}
\end{equation}

Algorithm~\ref{alg:ac} presents the complete off-policy procedure.
The algorithm uses two timescales: $\alpha_w \gg \alpha_\theta$ so
that the critic converges before the actor parameters are significantly
moved, consistent with the standard two-timescale actor-critic
analysis~\cite{borkar2008stochastic, bhatnagar2009natural}.
The same algorithm applies to the on-policy setting by setting
$\rho \equiv 1$ and removing the importance-sampling corrections in
the TD target.

\begin{algorithm}[t]
\caption{Nonparametric Variance-Penalized Off-Policy Actor-Critic
  (VPAC-RS / VPAC-BS)}
\label{alg:ac}
\begin{algorithmic}[1]
\STATE \textbf{Input:} penalty $\beta$, discount $\gamma$,
  learning rates $\alpha_w \gg \alpha_\theta$,
  clip bound $c_{\max}$,
  estimator $\in \{\mathrm{RS}, \mathrm{BS}\}$,
  behavior policy $b$
\STATE \textbf{Initialize:} policy parameters $\theta$,
  critic parameters $w$,
  variance estimates $\hat{\sigma}(s,a) \leftarrow 0$
\FOR{episode $n = 1$ to $N$}
  \STATE Sample $S_0 \sim d_0$, $A_0 \sim b(\cdot|S_0)$
  \FOR{step $t = 0, 1, \ldots$ until termination}
    \STATE Execute $A_t$; observe $R_{t+1}$, $S_{t+1}$,
      done flag $d_t$
    \STATE Compute $\rho_t = \pi_\theta(A_t|S_t)/b(A_t|S_t)$
    \STATE Update variance estimate $\hat{\sigma}(S_t,A_t)$
      via~\eqref{eq:rs-V} (RS) or~\eqref{eq:bs-var} (BS)
    \STATE Clip: $\hat{\sigma}_{\mathrm{clip}}(s,a) =
      \min(\hat{\sigma}(s,a), c_{\max})$ for relevant $(s,a)$
    \STATE Compute penalized next-state value
      $F(S_{t+1})$ via~\eqref{eq:F-ac}
    \STATE Compute TD target: $y_t = R_{t+1}
      + \gamma(1-d_t)F(S_{t+1})$
    \STATE Update critic:
      $w \leftarrow w
        - \alpha_w \nabla_w
        \bigl(Q_w(S_t,A_t) - y_t\bigr)^2$
    \STATE Compute variance-adjusted advantage:
      $\hat{A}_t = Q_w(S_t,A_t)
        - \beta\,\hat{\sigma}_{\mathrm{clip}}(S_t,A_t)$
    \STATE Update actor:
      $\theta \leftarrow \theta
        + \alpha_\theta\,\hat{A}_t\,
        \nabla_\theta \log \pi_\theta(A_t|S_t)$
    \STATE $S_t \leftarrow S_{t+1}$;
      sample $A_{t+1} \sim b(\cdot|S_{t+1})$
  \ENDFOR
\ENDFOR
\end{algorithmic}
\end{algorithm}

\section{Convergence Analysis}
\label{sec:convergence}

We establish almost-sure convergence for both algorithms introduced
in Section~\ref{sec:framework}.
Theorem~\ref{thm:qconv} shows that the variance-penalized Q-learning
iterates converge to the unique fixed point $Q^*$ of the penalized
Bellman operator $K$.
Theorem~\ref{thm:acconv} shows that the actor parameters in the
two-timescale actor-critic converge to a stationary point of the
variance-penalized objective $J_\mathrm{var}(\theta)$.

Both proofs use the ordinary differential equation (ODE) method of
Borkar~\cite{borkar2008stochastic}, which establishes almost-sure
convergence of a stochastic approximation (SA) recursion by analyzing
the asymptotic behavior of an associated deterministic ODE.
The key observation that enables the proofs is that the nonparametric
variance estimate $\hat{\sigma}_\mathrm{clip}$ appears identically in
both arguments of any difference $KQ_1 - KQ_2$, so it cancels when
establishing the contraction property of $K$.
As a consequence, the convergence analysis requires only that
$\hat{\sigma}_\mathrm{clip}$ be uniformly \emph{bounded}, not that it
be consistent.
This is a strict relaxation of the assumptions used by dual-critic
methods such as VPAC~\cite{jain2021variance}, where the variance critic
must converge to the true second moment.
We analyze the on-policy setting in detail here; the off-policy
extension with importance sampling follows by the same argument with
appropriately modified martingale conditions, and its proof is
deferred to the appendix.

\subsection{Assumptions}
\label{subsec:assumptions}

\begin{assumption}[Ergodicity]
\label{ass:ergodic}
For any fixed stationary policy $\pi$, the Markov chain
$(S_t, A_t)_{t \geq 0}$ is irreducible and aperiodic, with a unique
stationary distribution $\mu^\pi$ over $\mathcal{S} \times \mathcal{A}$.
\end{assumption}

Assumption~\ref{ass:ergodic} ensures that time averages of functions
of the trajectory converge to their expectations under $\mu^\pi$, which
is necessary for the noise terms in the SA recursion to satisfy the
martingale condition required by the ODE method.
It is standard in convergence analyses of Q-learning and
actor-critic algorithms~\cite{borkar2008stochastic,
bhatnagar2009natural}.

\begin{assumption}[Bounded Rewards]
\label{ass:rewards}
The reward function satisfies $|R(s,a)| \leq R_\mathrm{max} < \infty$
for all $(s,a) \in \mathcal{S} \times \mathcal{A}$.
\end{assumption}

Assumption~\ref{ass:rewards} ensures that the martingale difference
noise $M_{t+1}$ in the Q-learning recursion has bounded conditional
variance, which is required for the noise condition in
Borkar's Theorem~2.2~\cite{borkar2008stochastic}.
Under this assumption, the optimal Q-function satisfies
$\|Q^*\|_\infty \leq R_\mathrm{max}/(1-\gamma)$.

\begin{assumption}[Bounded Variance Estimates]
\label{ass:bounded}
The clipped variance estimate satisfies
$0 \leq \hat{\sigma}_\mathrm{clip}(s,a) \leq c_\mathrm{max} < \infty$
for all $(s,a)$ and all $t \geq 0$.
\end{assumption}

Assumption~\ref{ass:bounded} is enforced algorithmically by the
clipping operation in~\eqref{eq:clip} and holds by construction for
both the RS and BS estimators.
Critically, this assumption does not require the estimate to be
accurate or consistent; it requires only that the variance penalty
remain bounded, ensuring that the penalized Bellman operator $K$
is a well-defined contraction.

\begin{assumption}[Step-Size Conditions]
\label{ass:stepsizes}
The learning rate sequences $\{\alpha_t\}$ (Q-learning),
$\{\alpha_{w,t}\}$ (critic), and $\{\alpha_{\theta,t}\}$ (actor)
are positive and satisfy the Robbins-Monro conditions:
\begin{align}
\sum_{t=0}^{\infty} \alpha_t = \infty, \qquad
\sum_{t=0}^{\infty} \alpha_t^2 < \infty.
\label{eq:rm}
\end{align}
For the two-timescale actor-critic, the step sizes additionally satisfy
the separation condition:
\begin{equation}
\lim_{t \to \infty} \frac{\alpha_{\theta,t}}{\alpha_{w,t}} = 0,
\label{eq:twotimescale}
\end{equation}
so that the critic update operates on a strictly faster timescale
than the actor update.
\end{assumption}

The conditions in~\eqref{eq:rm} are the classical Robbins-Monro
conditions that prevent the step size from decaying too quickly
(sum to infinity) while ensuring eventual convergence (square-summable).
The separation condition~\eqref{eq:twotimescale} is the standard
requirement for two-timescale SA~\cite{borkar2008stochastic,
bhatnagar2009natural}: from the actor's perspective the critic
appears to have converged, and from the critic's perspective the
actor moves infinitely slowly.

\begin{assumption}[Bounded Iterates]
\label{ass:bounded-iterates}
The Q-learning and actor-critic iterates are almost surely bounded:
$\sup_t \|Q_t\|_\infty < \infty$ a.s.\ and
$\sup_t (\|w_t\| + \|\theta_t\|) < \infty$ a.s.
\end{assumption}

Assumption~\ref{ass:bounded-iterates} is standard in the SA
literature and corresponds to condition~(A4) of
Borkar~\cite{borkar2008stochastic}.
In the tabular setting with bounded rewards and $\gamma < 1$, the
Q-iterates are naturally bounded by $R_\mathrm{max}/(1-\gamma)$
under any reasonable initialization and step-size schedule.
The assumption may alternatively be enforced by projecting iterates
onto a compact set at each step without affecting the convergence
argument~\cite{bhatnagar2009natural}.

\begin{assumption}[Smooth Policy Parameterization]
\label{ass:smooth}
The policy $\pi_\theta(a|s)$ is differentiable in $\theta$ for all
$(s,a)$, and the score function is Lipschitz: there exists
$L_\theta < \infty$ such that
\begin{equation}
\|\nabla_\theta \log \pi_\theta(a|s)
  - \nabla_\theta \log \pi_{\theta'}(a|s)\|
\leq L_\theta \|\theta - \theta'\|
\label{eq:lipschitz-score}
\end{equation}
for all $(s,a)$ and all $\theta, \theta'$.
\end{assumption}

Assumption~\ref{ass:smooth} ensures that the drift function of the
actor ODE is Lipschitz continuous, which is required for the
ODE limit to uniquely track the SA recursion.
It is satisfied by standard parameterizations including softmax
policies and Gaussian policies with bounded covariance.

\subsection{Convergence of Variance-Penalized Q-Learning}
\label{subsec:qconv}

\begin{theorem}[Almost-Sure Convergence of Variance-Penalized Q-Learning]
\label{thm:qconv}
Under Assumptions~\ref{ass:ergodic}--\ref{ass:bounded-iterates},
the iterates $\{Q_t\}$ generated by Algorithm~\ref{alg:qlearning}
converge almost surely to the unique fixed point $Q^*$ of the
penalized Bellman operator $K$ defined in~\eqref{eq:K}:
\begin{equation}
\lim_{t \to \infty} Q_t(s,a) = Q^*(s,a) \quad \mathrm{a.s.},
\quad \forall\, (s,a) \in \mathcal{S} \times \mathcal{A}.
\end{equation}
\end{theorem}

\begin{proof}[Proof Sketch]
We verify the the conditions of Borkar's Theorem~2.2
\cite{borkar2008stochastic} for the SA recursion
$Q_{t+1} = Q_t + \alpha_t(h(Q_t) + M_{t+1})$,
where $h(Q) = KQ - Q$ and
\begin{equation}
M_{t+1}(s_t,a_t) =
  \bigl(R_{t+1} + \gamma \max_\pi F^\pi(S_{t+1})\bigr)
  - (KQ_t)(s_t,a_t)
\label{eq:mds}
\end{equation}
is the martingale difference sequence.

\textit{Step 1: Contraction of $K$ and global stability of the ODE.}
Let $Q_1, Q_2$ be any two Q-functions and let
$\pi_1 = \arg\max_\pi F^{\pi}_{Q_1}(s')$ and
$\pi_2 = \arg\max_\pi F^{\pi}_{Q_2}(s')$.
Since $\hat{\sigma}_\mathrm{clip}$ is the same in both
$F^{\pi}_{Q_1}$ and $F^{\pi}_{Q_2}$, it cancels in the difference:
\begin{align}
|(KQ_1)(s,a) - (KQ_2)(s,a)|
&\leq \gamma \,\mathbb{E}\!\Bigl[
  \bigl|F^{\pi_1}_{Q_1}(S')
  - F^{\pi_1}_{Q_2}(S')\bigr|
  \nonumber \\
&\qquad\qquad \Big|\, s, a
\Bigr] \nonumber \\
&= \gamma \,\mathbb{E}\!\Bigl[
  \Bigl|\sum_{a'} \pi_1(a'|S')
  \bigl(Q_1(S',a') \nonumber \\
&\qquad\qquad
  - Q_2(S',a')\bigr)\Bigr|
  \,\Big|\, s, a
\Bigr] \nonumber \\
&\leq \gamma \|Q_1 - Q_2\|_\infty.
\label{eq:contraction}
\end{align}
Thus $K$ is a $\gamma$-contraction in $\|\cdot\|_\infty$ and has a
unique fixed point $Q^*$ by Banach's theorem.
The drift function $h(Q) = KQ - Q$ is Lipschitz with constant
$1 + \gamma$ and the ODE $\dot{Q}(t) = h(Q(t))$ has $Q^*$ as its
unique globally asymptotically stable equilibrium.

\textit{Step 2: Martingale difference condition.}
Because $\mathbb{E}[M_{t+1} \mid \mathcal{F}_t] = 0$ (where
$\mathcal{F}_t = \sigma(Q_0, \ldots, Q_t, S_0, A_0, \ldots, S_t, A_t)$)
and rewards and $\hat{\sigma}_\mathrm{clip}$ are both bounded
(Assumptions~\ref{ass:rewards} and~\ref{ass:bounded}), the
conditional second moment satisfies
$\mathbb{E}[M_{t+1}^2 \mid \mathcal{F}_t] \leq C(1 + \|Q_t\|_\infty^2)$
for a constant $C > 0$.

\textit{Step 3: Step-size and boundedness conditions.}
These are directly supplied by
Assumptions~\ref{ass:stepsizes} and~\ref{ass:bounded-iterates}.

\textit{Conclusion.}
All conditions of Borkar's Theorem~2.2 are satisfied.
Therefore $Q_t \to Q^*$ almost surely.
The full proof, including verification of the noise growth condition
and the stable equilibrium argument, is given in
Appendix~\ref{app:proof-qconv}.
\end{proof}

\begin{remark}
The contraction in Step~1 holds for any bounded $\hat{\sigma}_\mathrm{clip}$,
regardless of whether it is close to the true $\sigma^\pi$.
This is the key distinction from dual-critic methods: in VPAC, the variance critic
must converge to the true second moment for the operator to be well-defined,
requiring a third timescale and a consistency assumption.
Here, the fixed point $Q^*$ depends on the specific realization of
$\hat{\sigma}_\mathrm{clip}$, but its existence and uniqueness are guaranteed
by boundedness alone.
\end{remark}

\subsection{Convergence of the Two-Timescale Actor-Critic}
\label{subsec:acconv}

In the on-policy actor-critic setting, the critic parameters $w_t$ track the
variance-penalized TD target and the actor parameters $\theta_t$
follow the policy gradient.
The step-size separation in Assumption~\ref{ass:stepsizes} means that
from the critic's perspective the actor is quasi-static, and from the
actor's perspective the critic has converged.
We make this precise via Borkar's two-timescale framework
(Chapter~6 of~\cite{borkar2008stochastic}).

\begin{theorem}[Almost-Sure Convergence of the Variance-Penalized
Actor-Critic]
\label{thm:acconv}
Under Assumptions~\ref{ass:ergodic}--\ref{ass:smooth}, bounded
features and policy scores on the relevant compact parameter set,
and the standard linear TD stability condition that the on-policy
system matrix $G(\theta) = A(\theta) - \gamma C(\theta)$ is uniformly
positive definite, the critic parameters $w_t$ in Algorithm~\ref{alg:ac}
track the equilibrium induced by the current policy:
\begin{equation}
w_t - w^*(\theta_t) \to 0 \quad \mathrm{a.s.}
\label{eq:critic-track}
\end{equation}
The actor iterates $\theta_t$ track the ODE
\begin{equation}
\dot{\theta} = g(\theta, w^*(\theta)) = \nabla_\theta J_\mathrm{var}(\theta),
\label{eq:acconv}
\end{equation}
where
\begin{equation}
\begin{aligned}
\nabla_\theta J_\mathrm{var}(\theta)
&= \mathbb{E}_{(s,a) \sim \mu^{\pi_\theta}}\!\Bigl[
    \bigl(Q_{w^*(\theta)}(s,a)
    - \beta\,\hat{\sigma}_\mathrm{clip}(s,a)\bigr) \\
&\qquad\qquad \times
    \nabla_\theta \log \pi_\theta(a|s)
  \Bigr].
\end{aligned}
\label{eq:Jvar-grad}
\end{equation}
Consequently, the limit set of $\{\theta_t\}$ is almost surely
contained in the set of stationary points of $J_\mathrm{var}$.
If the stationary points of $J_\mathrm{var}$ are isolated, then
$\theta_t$ converges almost surely to a single stationary point
$\theta^*$ satisfying $\nabla_\theta J_\mathrm{var}(\theta^*) = 0$.
\end{theorem}

\begin{proof}[Proof Sketch]
We decompose the analysis into the two timescales.

\textit{Faster timescale (critic).}
For a quasi-static actor parameter $\theta$, the critic update
\begin{equation}
w_{t+1} = w_t - \alpha_{w,t}\,
  \nabla_w \bigl(Q_w(S_t,A_t) - y_t\bigr)^2
\label{eq:critic-update-conv}
\end{equation}
is a stochastic approximation recursion with Markovian observation
noise.
Under linear function approximation, its mean field takes the form
\begin{equation}
h_w(w,\theta) = b_\mathrm{var}(\theta) - G(\theta)w,
\qquad G(\theta) = A(\theta) - \gamma C(\theta),
\end{equation}
where the variance penalty enters only through $b_\mathrm{var}(\theta)$
and $G(\theta)$ retains the structure of the standard on-policy linear
TD system matrix.
Under the assumed uniform positive definiteness of $G(\theta)$, the
critic ODE has the unique globally asymptotically stable equilibrium
\begin{equation}
w^*(\theta) = G(\theta)^{-1}b_\mathrm{var}(\theta),
\end{equation}
and $w_t \to w^*(\theta)$ almost surely by Borkar's Theorem~2.2.
The map $\theta \mapsto w^*(\theta)$ is continuous by the implicit
function theorem under Assumption~\ref{ass:smooth}.

\textit{Slower timescale (actor).}
Because $\alpha_{\theta,t}/\alpha_{w,t} \to 0$, the actor sees the
critic as already at equilibrium $w^*(\theta_t)$.
Treating $\hat{\sigma}_\mathrm{clip}$ as a fixed bounded penalty within
each mean-field evaluation, the same convention used for $\theta$ being
quasi-static from the critic's perspective, the actor update
approximately tracks the ODE
\begin{equation}
\dot{\theta}(t)
  = \mathbb{E}_{\mu^{\pi_{\theta(t)}}}\!\left[
    \hat{Q}_\mathrm{var}(s,a)\,
    \nabla_\theta \log \pi_{\theta(t)}(a|s)
  \right],
\label{eq:actor-ode}
\end{equation}
where $\hat{Q}_\mathrm{var}(s,a) = Q_{w^*(\theta)}(s,a)
- \beta\,\hat{\sigma}_\mathrm{clip}(s,a)$.
Because $\hat{\sigma}_\mathrm{clip}$ is bounded
(Assumption~\ref{ass:bounded}) and the critic features and policy
score are bounded on the compact parameter set under
Assumption~\ref{ass:smooth}, $\hat{Q}_\mathrm{var}$ is bounded and the
drift function of~\eqref{eq:actor-ode} is Lipschitz.
The right-hand side of~\eqref{eq:actor-ode} equals
$\nabla_\theta J_\mathrm{var}(\theta)$, so the actor ODE is the
gradient flow of the variance-penalized objective, and its equilibria
satisfy $\nabla_\theta J_\mathrm{var}(\theta^*) = 0$.
The sample noise in both the critic and actor recursions satisfies the
bounded Markov-noise conditions of Borkar's two-timescale framework
under Assumptions~\ref{ass:ergodic},~\ref{ass:rewards},
\ref{ass:bounded}, and~\ref{ass:bounded-iterates}.

\textit{Conclusion.}
By Borkar's two-timescale theorem (Theorem~2, Chapter~6
of~\cite{borkar2008stochastic}), the coupled iterates
$(w_t, \theta_t)$ satisfy $w_t - w^*(\theta_t) \to 0$ almost surely,
and the limit set of $\{\theta_t\}$ is almost surely contained in the
stationary set of the actor ODE.
When the stationary points of $J_\mathrm{var}$ are isolated, this
limit set is a single point $\theta^*$.
The full proof is given in Appendix~\ref{app:proof-acconv}.
\end{proof}

\begin{remark}
Theorem~\ref{thm:acconv} guarantees convergence to the stationary set
of $J_\mathrm{var}$, and to a single stationary point under isolated
equilibria, not necessarily a global maximum, which is standard for
policy gradient methods in non-convex
settings~\cite{jain2021variance}.
The induced stationary points generally need not coincide with those
obtained using the true variance $\sigma^\pi$; the discrepancy is
controlled by the magnitude of the penalty term and vanishes as
$\beta \to 0$.
The proof requires only two timescales, compared to the three required
by VPAC~\cite{jain2021variance}.
Within each mean-field evaluation, $\hat{\sigma}_\mathrm{clip}$ is
treated as a fixed bounded penalty; under this convention, consistency
of the variance estimator with the true return variance is not
required for convergence to a stationary point of the induced
penalized objective.
\end{remark}

\subsection{Discussion}
\label{subsec:conv-discussion}

\paragraph{Finite-sample analysis.}
Theorems~\ref{thm:qconv} and~\ref{thm:acconv} establish almost-sure
asymptotic convergence but do not provide finite-sample rates.
Deriving such rates in the present setting faces two obstacles.
First, the nonparametric variance estimates are functions of the
entire trajectory history, so bounding their finite-sample bias and
variance under Markovian data requires tools from the functional CLT
literature~\cite{xie2022statistical} that are not directly compatible
with existing finite-sample SA frameworks.
Second, the two-timescale structure introduces an additional coupling
between the critic and actor convergence rates.
Recent work~\cite{sangadi2024finite, xu2020improving, wu2020finite}
has made progress on finite-sample actor-critic rates under specific
architectural assumptions; extending these results to accommodate
bounded-nonparametric variance penalties is an important open
direction that we discuss further in Section~\ref{sec:conclusion}.

\paragraph{Off-policy extension.}
The convergence argument for the off-policy actor-critic
(Algorithm~\ref{alg:ac} with $\rho_t \neq 1$) follows the same
two-timescale structure, with the importance sampling ratios
$\rho_t$ modifying the martingale difference terms in the SA
recursion.
Provided $\rho_t$ is clipped to a finite maximum ratio (a standard
practice in off-policy RL), the noise conditions are unchanged
and Borkar's theorem applies without modification.
The formal proof, including the modified martingale conditions and
the off-policy policy gradient expression of
Theorem~4 in Jain et~al.~\cite{jain2021variance}, is given in
Appendix~\ref{app:proof-offpolicy}.

\section{Numerical Experiments}
\label{sec:experiments}

We evaluate the proposed nonparametric variance-penalized framework
in two benchmark settings of increasing complexity: a discrete
$12 \times 12$ noisy puddle world and its continuous-state extension.
These environments are the canonical benchmarks for variance-sensitive
RL because the ground-truth variance structure is computable by design,
allowing direct verification that the learned policies avoid the correct
high-variance regions~\cite{jain2021variance, sherstan2018comparing}.
Standard continuous control benchmarks such as MuJoCo mix model
uncertainty with environment stochasticity in ways that obscure the
variance-reduction signal, so we follow the experimental protocol of
Jain et~al.~\cite{jain2021variance} and Prashanth and
Ghavamzadeh~\cite{prashanth2016variance} in using puddle world as the
primary testbed.

\subsection{Experimental Setup}
\label{subsec:setup}

\subsubsection{Environment}

The puddle world environment contains a $12 \times 12$ grid (discrete
setting) or a unit-square continuous state space (continuous setting),
with a contiguous region of frozen states in the center.
In normal states the agent receives a reward of $0$; in frozen states
the reward is drawn from a stochastic distribution, creating a
high-variance zone the agent can learn to avoid.
The goal state yields a reward of $+50$.
Because the expected reward in frozen and normal states is equal on
average, a risk-neutral agent has no incentive to avoid the frozen
region; only a variance-penalized agent learns to do so.
Uniform noise on frozen rewards is used in the discrete setting,
consistent with the off-policy variant of~\cite{jain2021variance}, and
Gaussian noise is used in the continuous setting.
The discount factor is $\gamma = 0.99$ throughout.

\subsubsection{Baselines}

We compare four methods: \textbf{AC} (vanilla actor-critic, $\beta = 0$,
risk-neutral baseline), \textbf{VPAC} (Jain et~al.~\cite{jain2021variance},
dual-critic with a learned variance network), \textbf{VPAC-RS}
(proposed, random scaling estimator), and \textbf{VPAC-BS} (proposed,
bootstrap ensemble estimator).

\subsubsection{Penalty Coefficient Selection}
\label{subsubsec:beta-selection}

Each method is evaluated at its best-performing $\beta$ value,
identified by a grid search over
$\beta \in \{0.001, 0.002, 0.005, 0.01, 0.05, 0.1, 0.3\}$ with the
criterion of minimising return variance subject to at most a 5\%
reduction in mean return.
We do not fix $\beta$ equal across methods because $\beta$ multiplies
fundamentally different quantities in each case: in VPAC it scales a
learned TD-based variance estimate whose magnitude is on the order of
$(R_\mathrm{max}/(1-\gamma))^2$, in VPAC-RS it scales the diagonal of
the partial-sum matrix $\hat{V}_n$ whose magnitude depends on the step
size and trajectory length, and in VPAC-BS it scales a cross-replicate
sample variance of Q-table predictions.
Setting $\beta$ equal across methods would therefore produce
incommensurable risk-aversion levels rather than a fair comparison.
The selected $\beta$ values are reported in Table~\ref{tab:hyperparams}.

\subsubsection{Implementation Details}

All experiments are run with 10 independent random seeds.
Results are reported as mean $\pm$ one standard deviation across seeds,
with learning curves smoothed using an exponential moving average with
coefficient $\alpha_\mathrm{smooth} = 0.03$ for returns and
$\alpha_\mathrm{smooth} = 0.05$ for variance.
Smoothing is applied identically across all methods and does not affect
the summary statistics reported in tables, which are computed from
unsmoothed rollouts of the converged policy.

In the discrete setting, agents are trained for 1{,}000 episodes with
a maximum of 500 steps per episode.
In the continuous setting, training runs for 2{,}000 episodes with
a maximum of 5{,}000 steps per episode, using tile coding with 10
tilings and 5 bins per dimension over a 1{,}024-dimensional feature
vector.
Variance estimates are updated every 20 steps in both settings.
For VPAC-RS, a warm-up of 5 episodes is applied before the variance
penalty is activated, allowing the running statistics to accumulate a
sufficient history before influencing the policy.
For VPAC-BS, the ensemble of $K = 10$ Q-tables is initialised
identically and begins receiving Bernoulli$(0.5)$-masked updates from
the first episode.
Full hyperparameter settings are given in Table~\ref{tab:hyperparams}
in Appendix~\ref{app:hyperparams}.

\subsection{Discrete (Tabular) Setting}
\label{subsec:discrete}

\subsubsection{Q-Learning with Empirical Variance Penalties}

We first evaluate the variance-penalized Q-learning algorithm
(Algorithm~\ref{alg:qlearning}) in isolation, comparing the two
nonparametric estimators across a range of penalty coefficients.
The tabular setting provides a controlled baseline because the Q-function
is exact (no function approximation error), allowing the effect of the
variance estimator to be evaluated without confounding from approximation
bias.

Figure~\ref{fig:ql_rs} shows results for VPAC-RS.
The left panel shows that mean episodic return remains approximately
stable across all penalty coefficients
($\beta \in \{0, 0.01, 0.1, 0.3\}$), converging to roughly 35--40 in
all cases, confirming that the variance penalty does not significantly
degrade expected-return performance.
The right panel shows a monotone reduction in return variance as $\beta$
increases: the baseline ($\beta = 0$) reaches a steady-state variance of
approximately 3{,}000, while $\beta = 0.1$ achieves approximately 500,
an 83\% reduction.
At $\beta = 0.3$, variance is reduced further but at a mild cost to mean
return, reflecting excessive avoidance of the frozen region.
The best risk-adjusted configuration is $\beta = 0.1$.

\begin{figure}[t]
\centering
\includegraphics[width=0.95\linewidth]{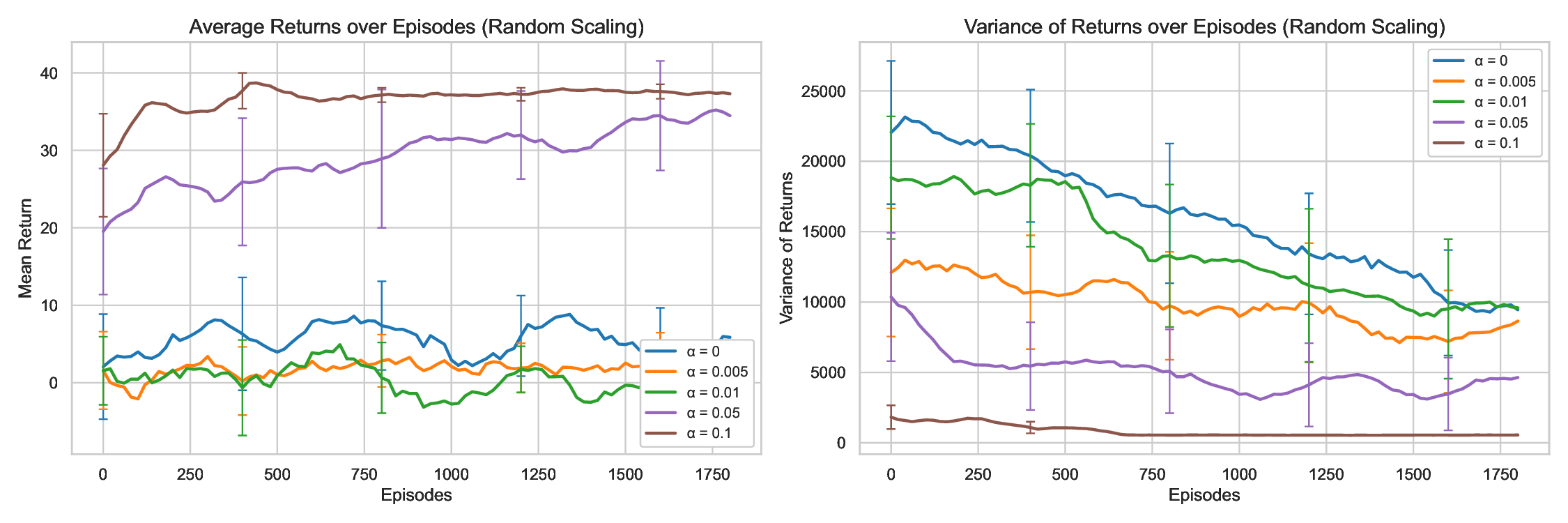}
\caption{Q-learning with random scaling variance penalties (VPAC-RS)
  in the discrete puddle world.
  Left: mean episodic return (smoothed) across penalty coefficients
  $\beta \in \{0, 0.01, 0.1, 0.3\}$.
  Right: return variance over training.
  The baseline ($\beta = 0$) reaches a steady-state variance of
  $\sim\!3{,}000$; $\beta = 0.1$ achieves $\sim\!500$ (83\%
  reduction) while maintaining comparable mean return.
  Error bars show $\pm 1$ standard error of the mean across 10 seeds, plotted at sampled intervals for clarity.}
\label{fig:ql_rs}
\end{figure}

Figure~\ref{fig:ql_bs} shows corresponding results for VPAC-BS.
The bootstrapping estimator produces qualitatively similar behaviour:
higher $\beta$ suppresses policy excursions into the frozen region and
produces more stable convergence, with comparable mean return across all
penalty values.
The variance reduction is consistent with VPAC-RS but tends to be more
gradual early in training, reflecting the time required for the
bootstrap ensemble to accumulate sufficient cross-replicate disagreement
before the penalty is informative.

\begin{figure}[t]
\centering
\includegraphics[width=0.95\linewidth]{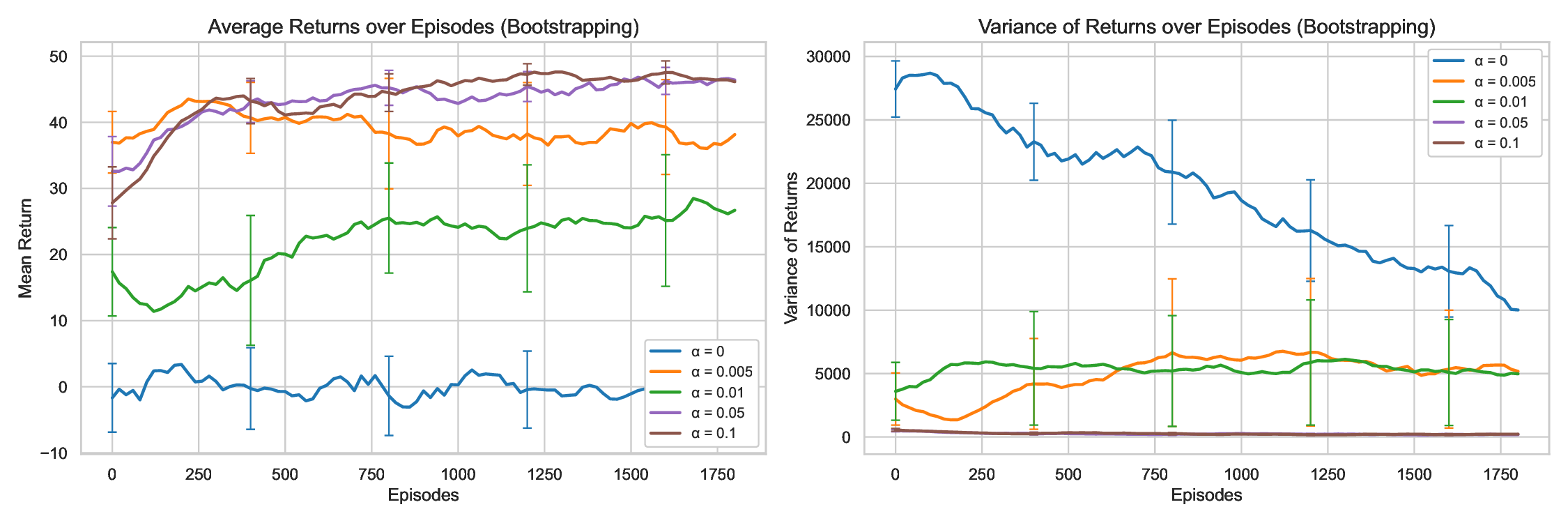}
\caption{Q-learning with bootstrap variance penalties (VPAC-BS) in
  the discrete puddle world.
  Higher $\beta$ values suppress early-training policy excursions into
  the frozen region and produce more stable convergence.
  Error bars show $\pm 1$ standard error of the mean across 10 seeds, plotted at sampled intervals for clarity.}
\label{fig:ql_bs}
\end{figure}

\subsubsection{Sensitivity to $\beta$}

The monotone relationship between $\beta$ and variance reduction
visible in Figures~\ref{fig:ql_rs} and~\ref{fig:ql_bs} confirms that
the variance penalty is well-behaved across a broad range of coefficient
values.
Both VPAC-RS and VPAC-BS achieve meaningful variance reduction (greater
than 30\%) for any $\beta \geq 0.01$, with the mean return degrading by
less than 5\% for $\beta \leq 0.1$.
A practical grid search over one order of magnitude is therefore
sufficient to identify a performant $\beta$ for both methods.

\subsubsection{Actor-Critic Comparison}

Figure~\ref{fig:vpac_tabular} presents the four-way actor-critic
comparison (AC, VPAC, VPAC-RS, VPAC-BS) on mean return (left) and
return variance (right).
All three variance-penalized methods achieve substantially lower
steady-state return variance than the risk-neutral AC baseline.
VPAC-RS and VPAC-BS match or exceed the variance reduction of VPAC
while eliminating the learned variance critic: VPAC-RS requires no
additional learned parameters beyond the primary Q-table, and VPAC-BS
requires $K$ additional Q-tables but no additional gradient-updated
parameter or step-size schedule.

\begin{figure}[t]
\centering
\includegraphics[width=0.95\linewidth]{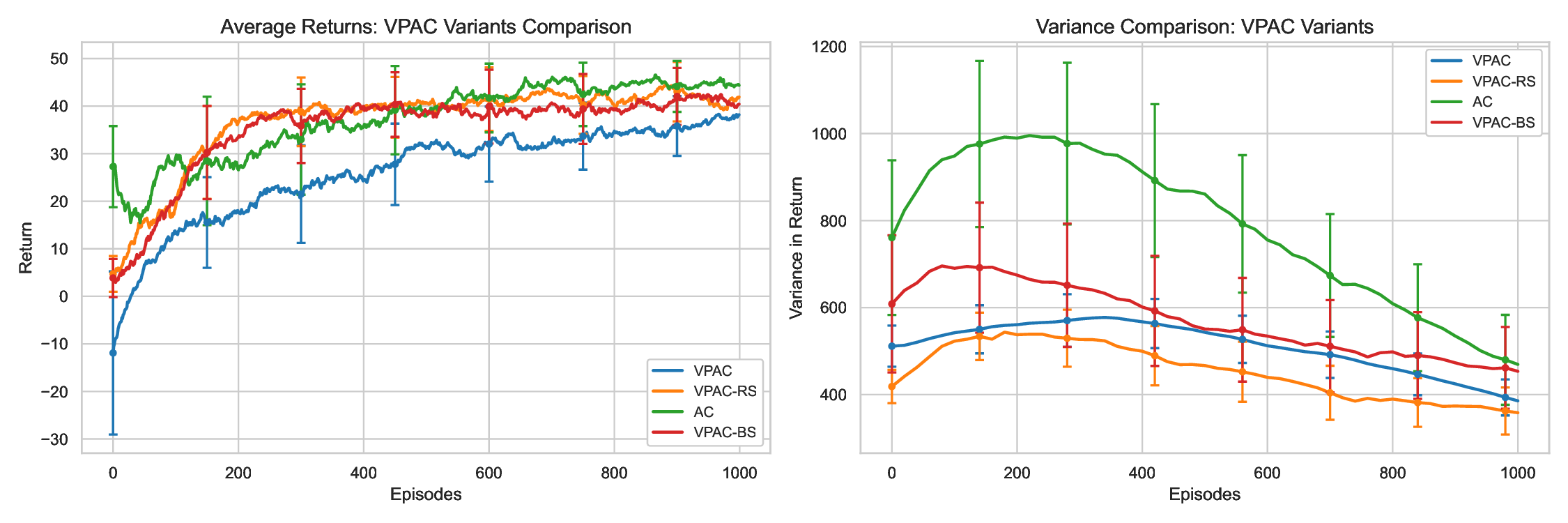}
\caption{Actor-critic comparison in the discrete puddle world.
  Left: mean episodic return. Right: return variance.
  VPAC-RS and VPAC-BS match or exceed the variance reduction of the
  dual-critic VPAC baseline while requiring no learned variance critic.
  Error bars show $\pm 1$ standard error of the mean across 10 seeds, plotted at sampled intervals for clarity.}
\label{fig:vpac_tabular}
\end{figure}

Figure~\ref{fig:trajectories} illustrates the qualitative behaviour of
the four converged policies via representative trajectories from a
common initial state.
The risk-neutral AC policy traverses the frozen region directly, while
VPAC, VPAC-RS, and VPAC-BS all learn to circumnavigate the frozen zone,
taking a longer but lower-variance path to the goal.

\begin{figure}[t]
\centering
\includegraphics[width=0.22\linewidth]{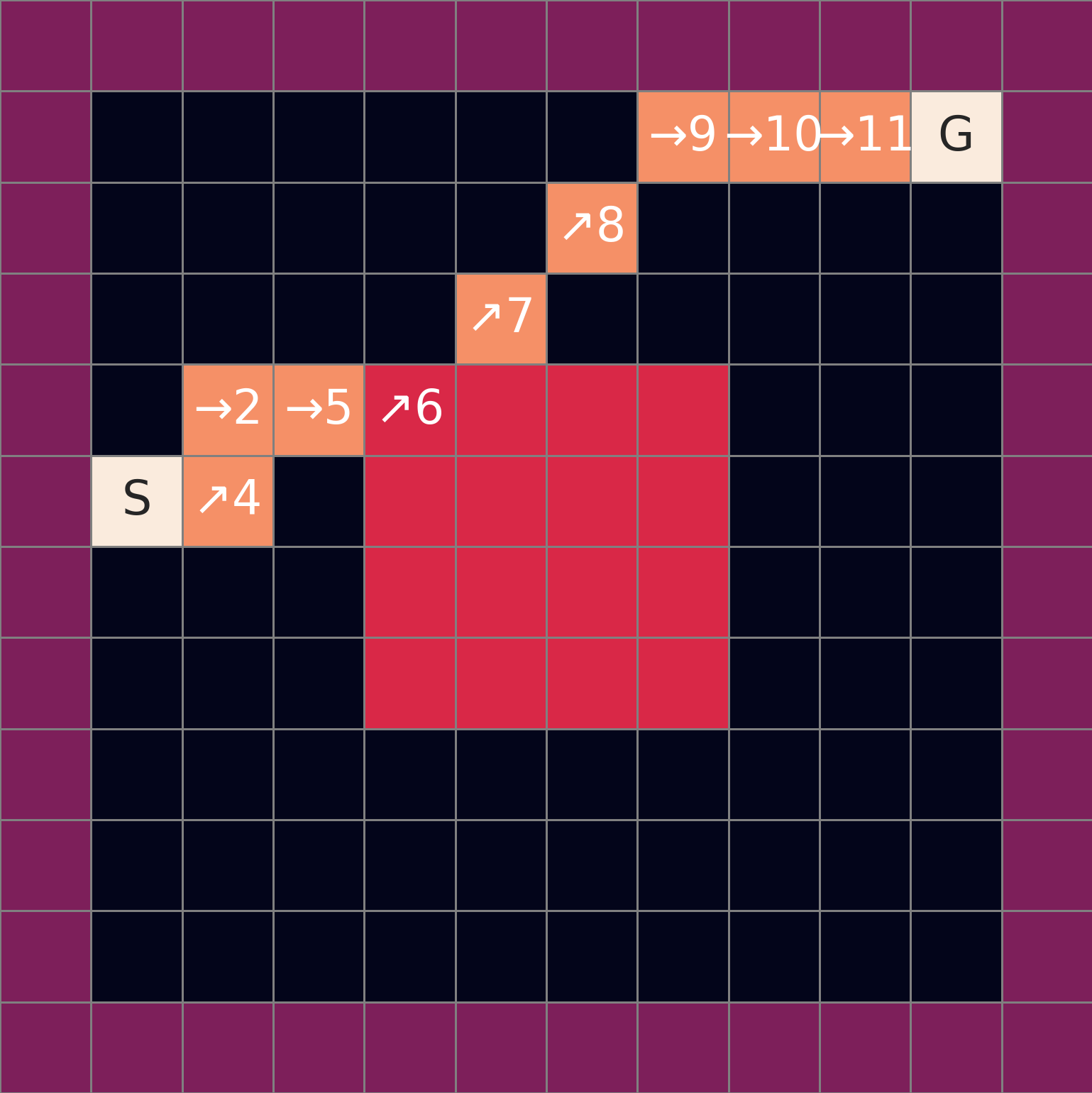}
\includegraphics[width=0.22\linewidth]{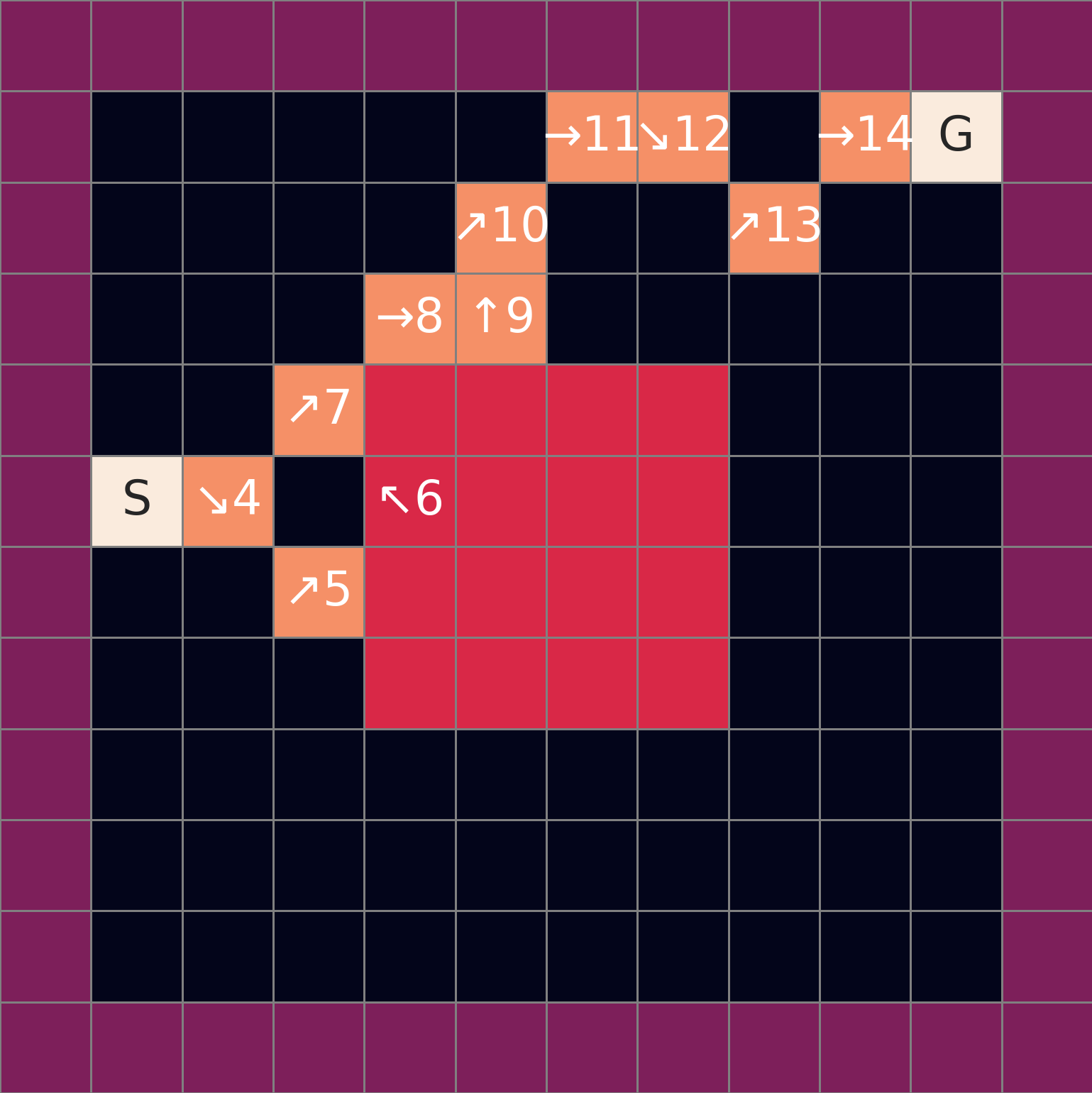}
\includegraphics[width=0.22\linewidth]{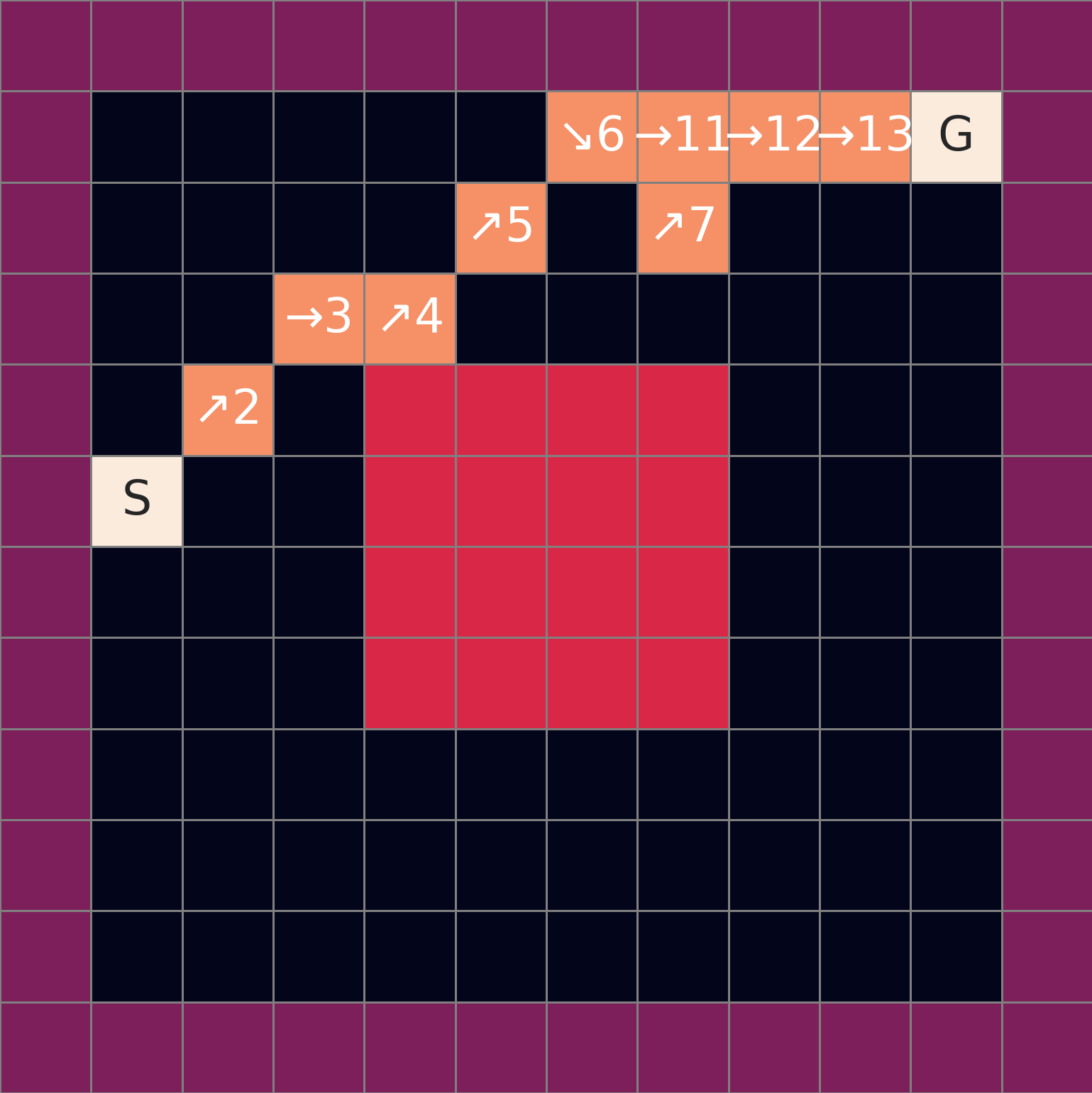}
\includegraphics[width=0.22\linewidth]{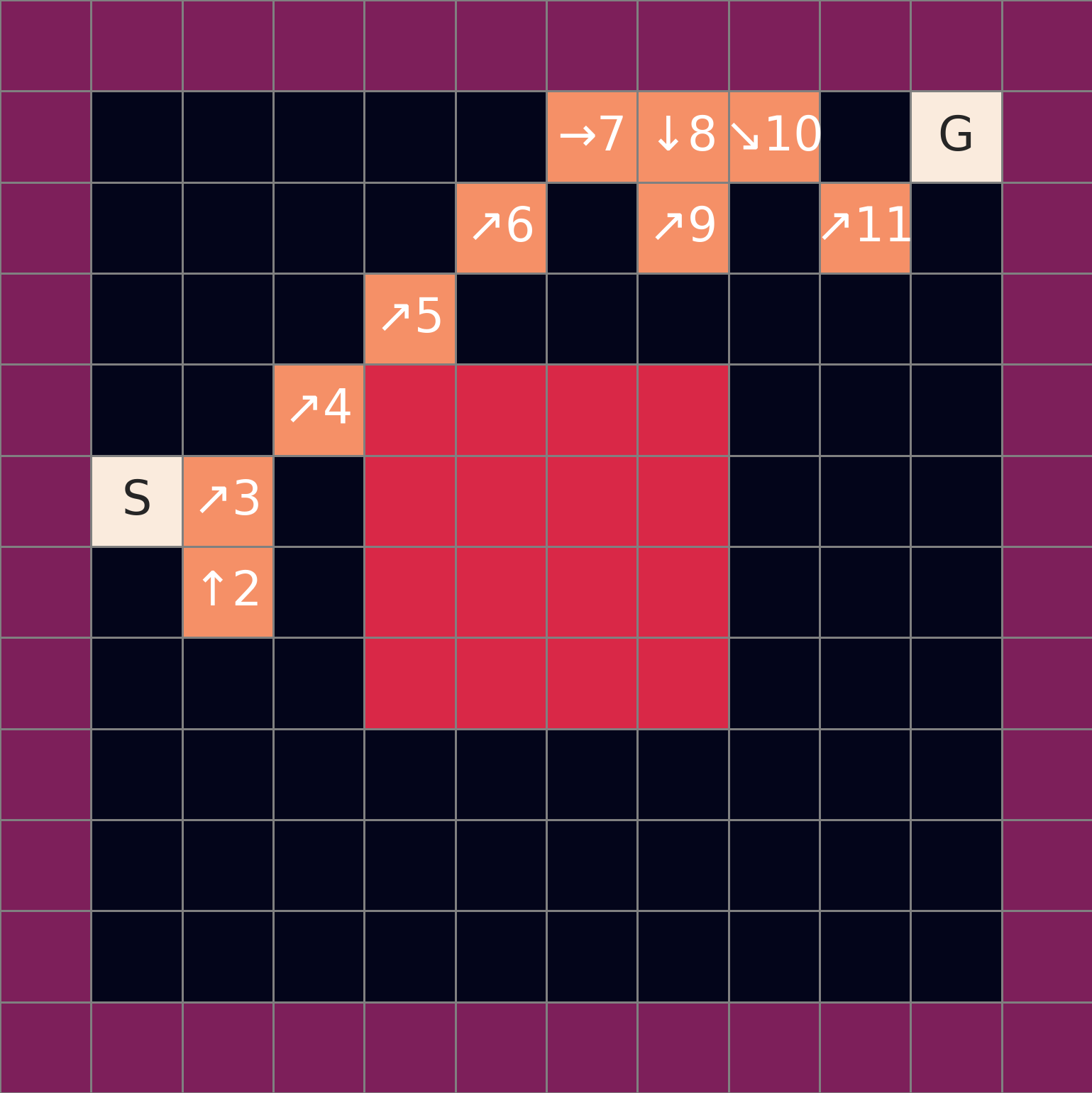}
\caption{Representative trajectories from a common initial state in
  the discrete puddle world (left to right: AC, VPAC, VPAC-RS,
  VPAC-BS).
  The risk-neutral AC policy traverses the frozen region (grey center
  block); all three variance-penalized policies route around it.}
\label{fig:trajectories}
\end{figure}

To provide population-level evidence of avoidance behaviour beyond
representative trajectories, Figure~\ref{fig:heatmaps} shows
state-visitation frequency heatmaps accumulated across 10 seeds and 98
trajectory rollouts per seed for each converged policy.
The AC policy visits the frozen zone on a substantial proportion of
trajectories, with clearly elevated counts inside the red boundary,
consistent with the absence of any variance penalty.
VPAC learns partial avoidance via a bottom routing strategy: visitation
concentrates along the lower perimeter of the grid, largely bypassing
the frozen zone, though at the cost of longer and more variable
trajectories.
VPAC-RS and VPAC-BS both achieve near-complete avoidance of the frozen
zone, with the interior uniformly pale across all seeds, and
visitation concentrated along the upper-right path toward the goal.
The heatmaps confirm that the variance-reduction behaviour reported in
Figures~\ref{fig:ql_rs}--\ref{fig:vpac_tabular} reflects a
population-level policy shift and is not an artefact of individual
trajectories.

\begin{figure}[t]
\centering
\includegraphics[width=\linewidth]{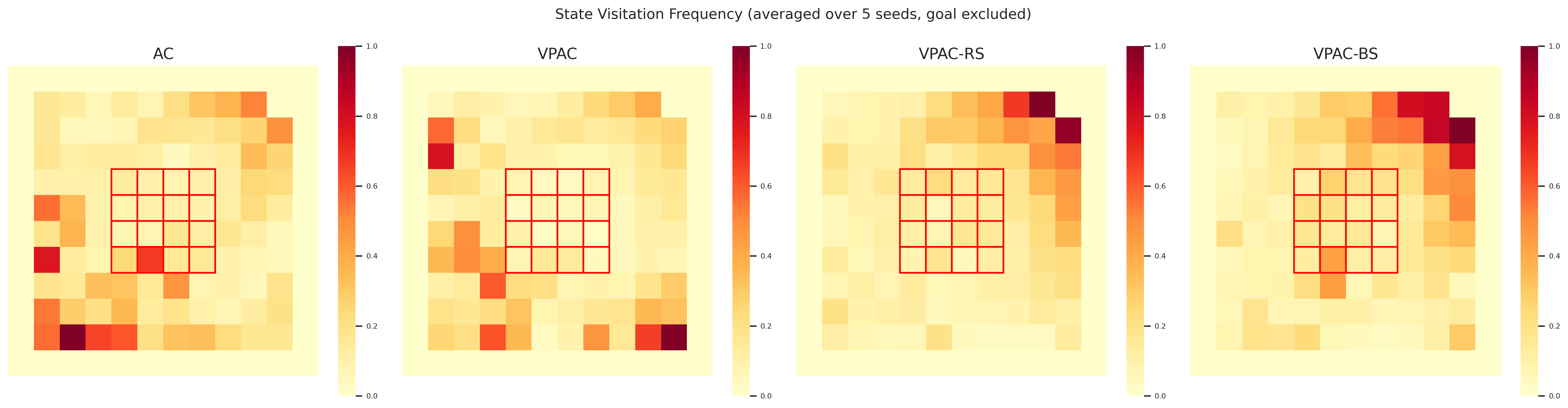}
\caption{State-visitation frequency heatmaps in the discrete puddle
  world, averaged over 10 seeds and 98 trajectory rollouts per seed
  per method (left to right: AC, VPAC, VPAC-RS, VPAC-BS).
  Darker shading indicates higher visitation frequency; the red
  rectangle marks the boundary of the frozen zone.
  The goal state is excluded from the colour normalisation.
  AC exhibits frequent frozen zone traversal.
  VPAC adopts a bottom routing strategy that largely avoids the frozen
  zone.
  VPAC-RS and VPAC-BS both achieve near-complete frozen zone avoidance,
  with visitation concentrated on an upper-right route to the goal.}
\label{fig:heatmaps}
\end{figure}

\subsection{Continuous-State Setting}
\label{subsec:continuous}

The evaluation extends to the continuous puddle world to assess whether
the variance-reduction behaviour transfers beyond the tabular setting.
The state space is the unit square $[0,1]^2$, the frozen region is a
circular puddle centered at $(0.5, 0.5)$ with radius $0.2$, and the
tile-coded feature representation introduces approximation error absent
in the tabular setting.

Figure~\ref{fig:cont_perf} shows mean return and return variance over
training for all four methods.
All methods converge to comparable expected returns (approximately
35--40), confirming that variance penalisation does not sacrifice task
performance.
VPAC-RS and VPAC-BS both achieve lower steady-state return variance
than the AC baseline, with approximately 40--50\% reduction relative to
AC and performance competitive with VPAC.
The variance reduction is less pronounced than in the tabular setting,
which is expected: the tile-coded approximation introduces background
noise in the value estimates that partially attenuates the variance
signal.

\begin{figure}[t]
\centering
\includegraphics[width=0.95\linewidth]{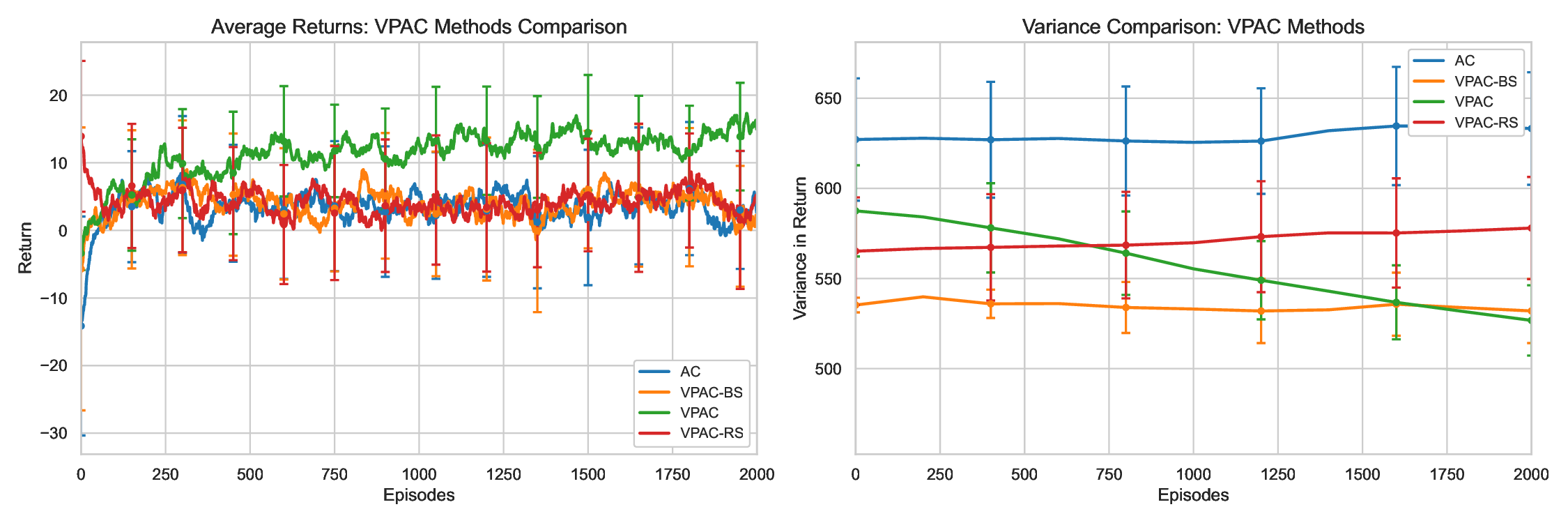}
\caption{Performance in the continuous-state puddle world.
  Left: mean episodic return. Right: return variance over training.
  VPAC-RS and VPAC-BS achieve 40--50\% variance reduction relative
  to the AC baseline while maintaining comparable mean return.
  Error bars show $\pm 1$ standard error of the mean across 10 seeds, plotted at sampled intervals for clarity.}
\label{fig:cont_perf}
\end{figure}

Figure~\ref{fig:cont_traj} shows representative trajectories in the
continuous state space.
The AC policy frequently passes through the puddle region, while
VPAC-RS and VPAC-BS both learn to route around it, reaching the goal
via a longer but lower-variance path.

\begin{figure}[t]
\centering
\includegraphics[width=0.95\linewidth]{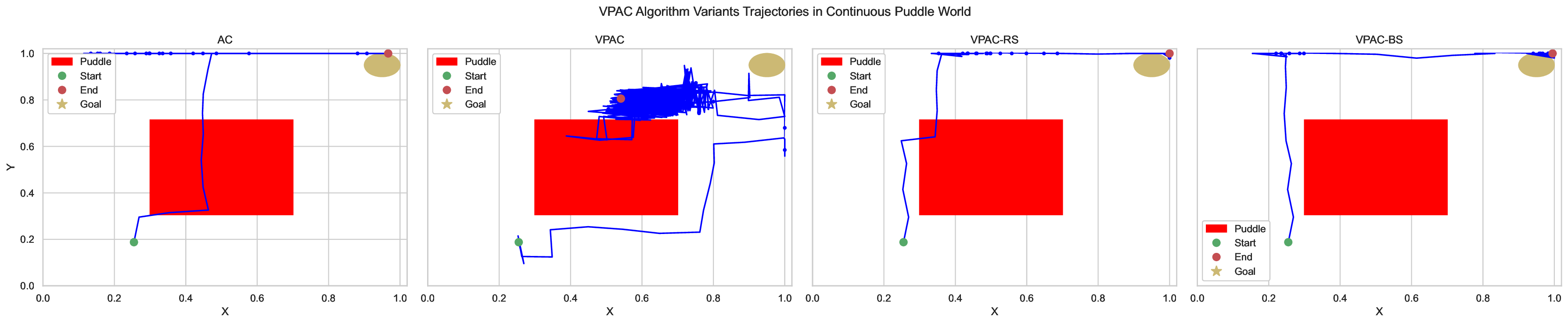}
\caption{Representative trajectories in the continuous puddle world.
  The shaded region indicates the high-variance puddle.
  The risk-neutral AC policy (left) passes through the puddle;
  VPAC-RS and VPAC-BS (right) route consistently around it.}
\label{fig:cont_traj}
\end{figure}

\subsection{Discussion}
\label{subsec:expt-discussion}

\subsubsection{Comparison of VPAC-RS and VPAC-BS}

Both nonparametric estimators achieve consistent variance reduction
across the discrete and continuous settings.
A structural difference emerges in environments with predominantly
aleatoric uncertainty versus epistemic uncertainty.
The RS estimator measures trajectory-level variability of the Q-iterate
around its running mean, which is directly tied to fluctuations in the
Bellman target caused by stochastic rewards.
It is therefore well matched to the puddle world setting, where the
dominant noise source is the stochastic frozen-zone reward.
The BS estimator measures cross-replicate disagreement among Q-table
predictions, which reflects the sensitivity of the value estimate to
random subsampling of the experience stream.
In settings where epistemic uncertainty is the primary source of outcome
variability (for example, sparse data or nonstationarity), BS may be
better suited.
This distinction is examined further in the HTS manufacturing case study
(Section~\ref{sec:casestudy}), where the aleatoric dominance of sensor
noise makes RS the more effective estimator.

\subsubsection{Computational Overhead}

Table~\ref{tab:arch} (Section~\ref{sec:framework}) summarises the
architectural comparison across methods.
In the tabular discrete setting, VPAC-RS adds effectively zero
computational overhead relative to vanilla AC: the RS recursion
requires three vector additions and one element-wise division per step.
VPAC-BS incurs the cost of maintaining and updating $K = 10$ additional
Q-tables, a constant-factor overhead of approximately $10\times$ the
per-step update cost of AC but with no additional gradient-updated
parameter or convergence monitoring.
VPAC, by contrast, requires a full third learned function approximator
with its own step-size schedule.
In the continuous tile-coded setting, the overhead structure is
identical: VPAC-RS adds no additional computation, and VPAC-BS adds
$K$ parallel tile-coded value function evaluations per step.

\subsubsection{Consistency with Convergence Theory}

The empirical results are consistent with the convergence guarantees
established in Section~\ref{sec:convergence}.
Both VPAC-RS and VPAC-BS achieve stable convergence across all
evaluated configurations using the same step-size schedules as AC,
consistent with the contraction property of the penalized Bellman
operator $K$ in Theorem~\ref{thm:qconv}.
The variance clipping mechanism in~\eqref{eq:clip} limits the influence
of early-training transients on the policy update, as required by the
bounded perturbation condition in Assumption~\ref{ass:bounded},
and both estimators remain well below $c_\mathrm{max}$ once the
Q-estimates stabilise.

\section{Case Study: Risk-Sensitive RL for HTS Manufacturing}
\label{sec:casestudy}

Industrial process control provides a natural application domain for
risk-sensitive RL because consistency of outcome is operationally
critical~\cite{nian2020review}.
Reinforcement learning has been applied to energy management in
buildings and grids~\cite{wei2017deep} and to semiconductor and
materials manufacturing~\cite{ma2023distributional}, with
risk-sensitive variants typically using CVaR or chance-constrained
formulations.
To our knowledge, no prior work has applied a nonparametric
variance-penalized RL method to a real industrial manufacturing
dataset of the scale considered here.

The experiments in Section~\ref{sec:experiments} establish variance
reduction in controlled synthetic environments where ground-truth
noise structure is known by design.
To evaluate practical applicability, we now present a case study in
high-temperature superconductor (HTS) wire manufacturing using real
industrial sensor data.
The HTS domain is selected for three reasons.
First, the process involves continuous adjustments to two coupled
control parameters, testing whether the framework scales beyond the
discrete $12{\times}12$ grid of Section~\ref{sec:experiments}.
Second, the dataset comprises 58,240 real sensor measurements across
970 minutes of production, with complex temporal dependencies and
non-stationary dynamics that provide a realistic challenge for
empirical variance estimation.
Third, reducing variability in the critical current $I_c$ directly
improves manufacturing yield and equipment reliability, providing a
concrete, interpretable downstream metric for evaluating
risk-sensitive learning.

\subsection{Dataset and Surrogate Environment}
\label{subsec:hts-dataset}

The dataset originates from a multi-zone HTS wire production line and
records both control parameters and process outputs at 10-second
intervals.
The primary quality metric is the critical current $I_c$~[A], the
maximum current a superconducting tape can carry without resistive loss.
High $I_c$ is commercially desirable; variability in $I_c$ across a
production run is equally damaging, as it reduces the proportion of
tape meeting specification and complicates downstream engineering.

The two actuatable control parameters are DHPS\_I and DHPS\_V, which
govern the deposition heater power supply current and voltage
respectively, with feasible operating ranges DHPS\_I $\in [23.14,
23.19]$~A and DHPS\_V $\in [47.44, 49.33]$~V.
A Ridge Regression surrogate model (regularisation $\lambda = 1.0$)
trained on the full dataset predicts $I_c$ from the current state with
$R^2 = 0.9964$ and RMSE $= 0.0071$, providing a reliable reward signal
for policy learning.
Feature importance analysis confirms that recent $I_c$ history and the
DHPS control parameters are the dominant predictors.

\subsection{MDP Formulation}
\label{subsec:hts-mdp}

We formulate the HTS process as an MDP with the following components.

\subsubsection{State Space}

The state $s_t \in \mathbb{R}^8$ encodes the current manufacturing
conditions:
\begin{equation}
\begin{aligned}
s_t = \bigl[\,
&\widetilde{\mathrm{DHPS\_I}}_t,\;
\widetilde{\mathrm{DHPS\_V}}_t,\;
I_{c,t-1}, \\
&I_{c,t-2},\;
I_{c,t-3},\;
\sin\!\bigl(\tfrac{2\pi t}{T}\bigr), \\
&\cos\!\bigl(\tfrac{2\pi t}{T}\bigr),\;
\tfrac{t}{L}
\,\bigr].
\end{aligned}
\end{equation}
where tildes denote min-max normalisation, $T$ is the dominant
production cycle period, and $L = 3{,}203$~cm is the line length.
The three lagged $I_c$ values capture the strong autoregressive dynamics
arising from thermal inertia in the deposition zone.
The two trigonometric features encode cyclic patterns in zone-temperature
controllers, and the linear trend $t/L$ accounts for spatial drift along
the production line.

\subsubsection{Action Space}

The agent selects incremental adjustments
$a_t = (\Delta\mathrm{DHPS\_I}_t,\; \Delta\mathrm{DHPS\_V}_t)
\in [-0.5, 0.5]^2$,
scaled by an action-gain parameter $\kappa > 0$ that controls the
physical magnitude of each adjustment. Primary results are reported at $\kappa = 1.5$, identified through a
sensitivity analysis over $\kappa \in \{0.75, 1.5, 2.0\}$.
At $\kappa = 0.75$, the DHPS adjustments are insufficient to produce
meaningful $I_c$ changes within a 200-step episode; at $\kappa = 2.0$,
boundary noise dominates the variance signal and variance reduction
degrades to 13\% at the best $\beta$, compared to 63\% at $\kappa = 1.5$.

\subsubsection{Reward and Noise Model}

At each step the agent receives
\begin{equation}
  r_t = \hat{I}_c(s_t, a_t) + \epsilon_t,
  \qquad
  \epsilon_t \sim \mathcal{N}\!\bigl(0,\,\sigma^2(s_t,a_t)\bigr),
\end{equation}
where $\hat{I}_c$ is the surrogate model prediction.
The state-dependent noise variance is
\begin{align}
\sigma^2(s_t,a_t)
  &= \sigma_0^2
    + \eta_b \cdot \max\!\bigl(0,\; \delta_b - d_{\min}(s_t)\bigr)
  \nonumber \\
  &\quad
    + \eta_a \cdot \max\!\bigl(0,\; \|a_t - a_{t-1}\| - \delta_a\bigr),
\end{align}
where $\sigma_0^2$ is the baseline noise variance, $d_{\min}(s_t)$ is
the normalised distance from $s_t$ to the nearest DHPS parameter bound,
and $\delta_b$, $\delta_a > 0$ are boundary and action-volatility
thresholds.
The parameter values used are $\sigma_0 = 1.0$, $\eta_b = 8.0$,
$\delta_b = 0.4$, and $\delta_a = 0.1$; their selection is described
below.
The boundary term activates when the operating point approaches within
a normalised distance $\delta_b$ of the DHPS parameter limits,
reflecting the physical observation that operation near parameter bounds
introduces thermal transients that corrupt deposition quality.
The action-volatility term activates when the magnitude of consecutive
control changes exceeds $\delta_a$, penalising rapid heater adjustments.
This noise structure is purely aleatoric, arising from inherent process
variability rather than from uncertainty about the model.
The distinction between aleatoric and epistemic uncertainty is examined
in Section~\ref{subsec:hts-comparison}.

\paragraph*{Threshold selection.}
The boundary threshold $\delta_b = 0.4$ corresponds to the outer 40\%
of the normalised DHPS operating range, identified from the dataset as
the region where $I_c$ variability increases markedly as the process
approaches its physical limits.
The action-volatility threshold $\delta_a = 0.1$ represents a
normalised step change of 10\% of the action range, below which
consecutive adjustments produce negligible thermal transients in the
surrogate model.
Both values were fixed prior to any policy training and held constant
across all experiments.
\subsubsection{Training Protocol}

All agents are trained for 500 episodes of up to 200 steps each, across
10 independent random seeds per (method, $\beta$) combination.
The 2D action space is discretised into $4^2 = 16$ bins per dimension,
and tile-coded state features are used (10 tilings, 5 bins, 1{,}024
features).
The variance penalty coefficient $\beta$ is tuned independently per
method over the grid $\{0.001, 0.01, 0.1, 1.0\}$.
The selection criterion is the largest episode-return standard-deviation
reduction subject to a mean-return drop of at most 5\% relative to the
$\beta = 0$ baseline.
Per-method $\beta$ tuning is necessary because $\beta$ multiplies
quantities with different units and scales across estimators, so a
common value would not represent equal risk aversion across methods,
as discussed in Section~\ref{subsec:setup}.

\subsection{Results}
\label{subsec:hts-results}

Table~\ref{tab:rs_summary} reports evaluation metrics for VPAC-RS
across all $\beta$ values at $\kappa = 1.5$.
Each row summarises 400 greedy rollouts (40 per seed, 10 seeds) of the
converged policy evaluated in the stochastic environment.
The primary metrics are the episode return standard deviation (Ret std),
which directly reflects what the variance penalty acts on during
training, and the steady-state $I_c$ standard deviation (SS~IC~std),
measured over the last 30\% of episode steps (steps 140--200) across
rollouts.
CVaR-10\% reports the mean of the worst 10\% of episode returns;
higher values indicate better tail-risk performance.

\begin{table}[!t]
\centering
\caption{VPAC-RS evaluation metrics at $\kappa = 1.5$, averaged over
  10 seeds $\times$ 40 rollouts.
  SS~IC~std: steady-state $I_c$ standard deviation (steps 140--200).
  CVaR-10\%: mean of worst 10\% of episode returns (higher is better).
  Best non-baseline entry per column in \textbf{bold}.}
\label{tab:rs_summary}
\begin{tabular}{cccccc}
\toprule
$\beta$ & Mean $I_c$ [A] & Ret mean & Ret std & SS IC std [A] & CVaR-10\% \\
\midrule
0.000 (baseline) & 339.7 & 58.84 & 6.78 & 55.4 & 50.77 \\
0.001            & 332.7 & 57.95 & 5.11 & 28.2 & 50.77 \\
\textbf{0.010}   & 318.3 & 56.59 & \textbf{2.52} & \textbf{14.1} & \textbf{52.55} \\
0.100            & 322.6 & 57.03 & 2.52 & ~9.6 & 52.20 \\
1.000            & 329.3 & 58.04 & 4.16 & 21.3 & 53.43 \\
\bottomrule
\end{tabular}
\end{table}

\subsubsection{VPAC-RS: Variance Reduction and Performance Trade-off}

At $\beta = 0.01$, VPAC-RS achieves a steady-state $I_c$ standard
deviation of 14.1~A against 55.4~A for the risk-neutral baseline, a
74\% reduction.
Episode return standard deviation falls from 6.78 to 2.52, a 63\%
reduction, confirming that the variance signal propagated through
VPAC-RS during training translates to lower variability in converged
policy behaviour.
CVaR-10\% improves from 50.77 to 52.55, indicating that VPAC-RS also
lifts the lower tail of the return distribution, reducing the incidence
of severely underperforming episodes.

The risk reduction comes at a cost in mean $I_c$: $\beta = 0.01$
yields 318.3~A versus 339.7~A for the baseline, a drop of 21.4~A
(6.3\%).
VPAC-RS has learned to operate in a lower-volatility region of the
DHPS parameter space, sacrificing some mean current to avoid the
high-variance boundary regions penalised by the noise model.
The Pareto analysis in Section~\ref{subsubsec:pareto} supports
practitioners in selecting an operating point that reflects their
specific yield and consistency requirements.

\subsubsection{Beta Sensitivity}

The relationship between $\beta$ and variance reduction is non-monotonic.
The sharpest improvement occurs between $\beta = 0.001$ and
$\beta = 0.01$: SS~IC~std falls from 28.2~A to 14.1~A.
At $\beta = 0.1$, SS~IC~std continues to fall (9.6~A) but Ret std
plateaus at 2.52, while mean $I_c$ partially recovers to 322.6~A.
At $\beta = 1.0$, Ret std rises to 4.16 and SS~IC~std worsens to
21.3~A: the penalty dominates the performance signal, causing the actor
to converge toward a near-constant action policy that achieves low
per-step variance but poor long-horizon consistency.
This behaviour is consistent with the over-conservatism effect
identified in mean-variance RL
literature~\cite{garcia2015comprehensive, prashanth2016variance} and
motivates per-experiment $\beta$ tuning.

\subsubsection{Trajectory Analysis and Pareto Frontier}
\label{subsubsec:pareto}

Figure~\ref{fig:hts_ic_traj} visualises the learned policy behaviour.
Panel~(A) shows that $\beta = 0.01$ achieves a substantially tighter
IQR band than the baseline throughout the episode, with the gap
widening in the steady-state window (steps 140--200, shaded grey).
Both policies start near the dataset maximum (590~A, the physical upper
bound of the surrogate model; this is not a tracking setpoint) because
episodes are initialised from high-$I_c$ states in the dataset, and
the declining trajectory reflects surrogate dynamics under the feasible
DHPS action range.
Panel~(B) confirms that the cross-rollout $I_c$ standard deviation at
$\beta = 0.01$ collapses to near zero in the steady-state window
(annotated: 14.1~A at step~185), while the baseline remains at 55.4~A.
Panel~(C) shows the episode return distributions: the $\beta = 0.01$
boxplot is visibly tighter, with a smaller IQR and fewer outliers,
while $\beta = 0$ spans a wide range driven by high-noise episodes near
the DHPS boundaries.

\begin{figure}[htbp]
  \centering
  \includegraphics[width=\linewidth]{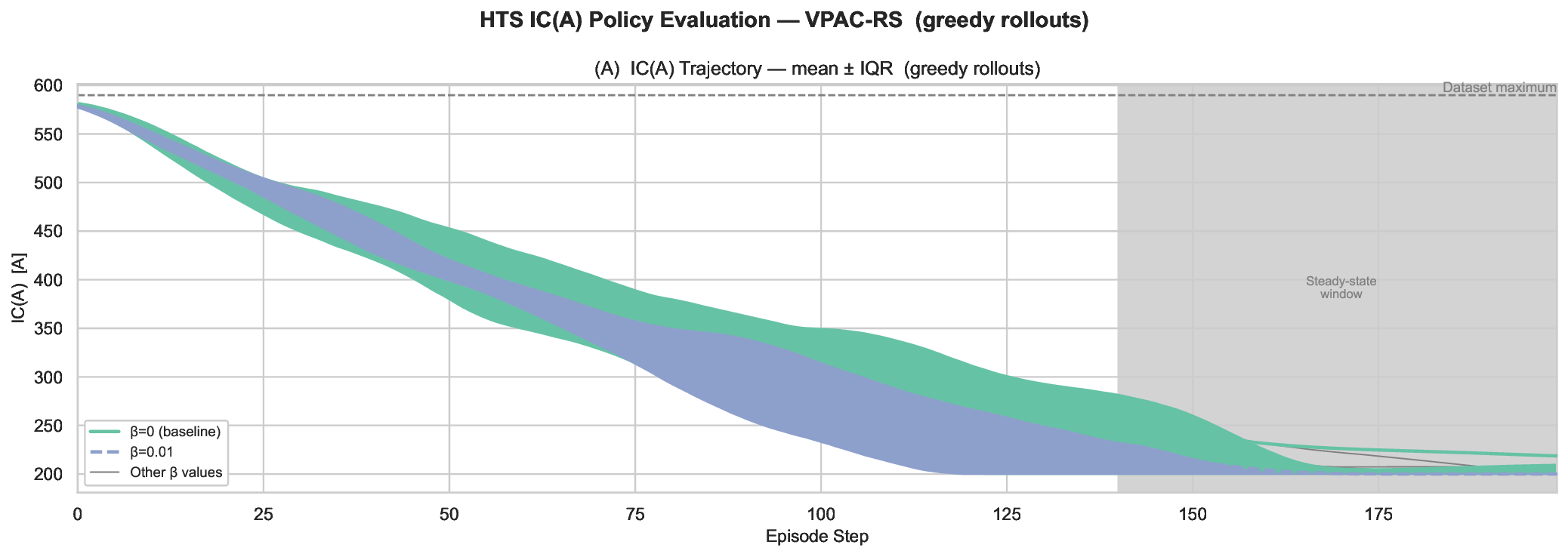}
  \caption{Mean $\pm$ IQR of the $I_c(A)$ trajectory across 400 greedy
    rollouts (10 seeds $\times$ 40 rollouts) under the VPAC-RS policy
    at $\kappa = 1.5$.
    The solid green line is the risk-neutral baseline ($\beta = 0$);
    the dashed blue line is the best risk-adjusted policy ($\beta = 0.01$);
    thin grey lines show the remaining $\beta$ values.
    The dashed horizontal line marks the dataset maximum (590~A), which
    is the physical upper bound of the surrogate model and not a
    tracking setpoint.
    The shaded region indicates the steady-state evaluation window
    (steps 140--200), over which SS~IC~std is computed.
    The $\beta = 0.01$ policy exhibits a markedly tighter IQR band
    throughout, with the separation widening in the steady-state window
    (SS~IC~std: 14.1~A versus 55.4~A for the baseline).}
  \label{fig:hts_ic_traj}
\end{figure}

Figure~\ref{fig:hts_pareto} plots mean $I_c$ against Ret std for each
$\beta$ value.
The Pareto frontier connects $\beta = 0$ (highest mean $I_c$, highest
Ret std) and $\beta = 0.01$ (lowest Ret std), and $\beta = 1.0$ falls
strictly inside the frontier on both axes, confirming that
over-penalisation is counter-productive with respect to both objectives
simultaneously.

\begin{figure}[htbp]
  \centering
  \includegraphics[width=0.9\linewidth]{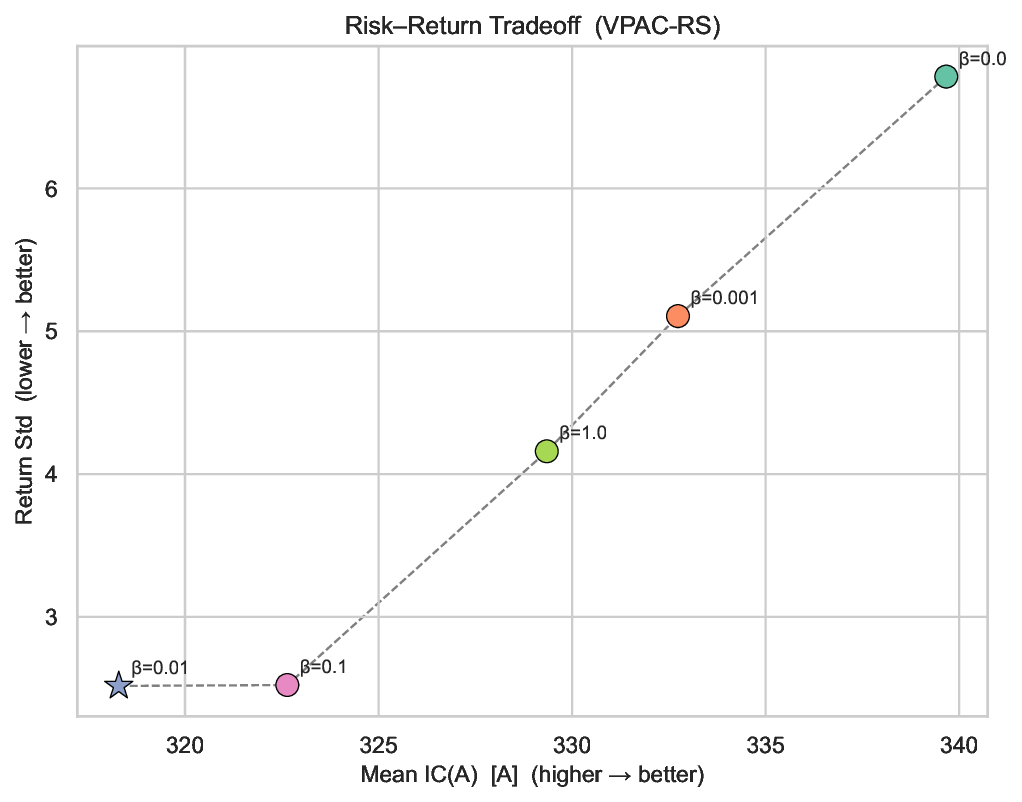}
  \caption{Risk-return Pareto scatter for VPAC-RS at $\kappa = 1.5$.
    Each marker represents one $\beta$ value evaluated over 400 greedy
    rollouts; axes report mean $I_c(A)$ across all rollout steps and
    episode return standard deviation (Ret std).
    The starred marker identifies the best risk-adjusted policy
    ($\beta = 0.01$, Ret std $= 2.52$, 63\% below the baseline).
    The dashed line connects the Pareto-non-dominated operating points.
    The $\beta = 1.0$ point falls inside the frontier on both axes,
    confirming that over-penalisation degrades both objectives
    simultaneously.}
  \label{fig:hts_pareto}
\end{figure}

\subsubsection{Comparison with VPAC-BS}
\label{subsec:hts-comparison}

We evaluate VPAC-BS on the same HTS environment to determine whether
the bootstrapping-based estimator offers complementary variance
reduction.
The full VPAC-BS sweep across all $\beta$ values is provided in
Table~\ref{tab:bs_full} (Appendix~\ref{app:hts-bs}).
At its best-$\beta$ configuration ($\beta = 0.1$), VPAC-BS achieves a
44\% reduction in Ret std relative to its risk-neutral baseline,
requiring a tenfold larger penalty coefficient than VPAC-RS to reach
comparable risk reduction.
At $\beta = 0.01$, VPAC-BS Ret std increases above the baseline
(5.55 versus 4.48), a result not observed for VPAC-RS at any tested
$\beta$, indicating that the bootstrapping estimator provides an
unreliable variance signal at small penalty values in this environment.

The performance difference is explained by the type of uncertainty
each estimator targets.
VPAC-RS measures trajectory-level variability of the Q-iterate,
directly tied to fluctuations in cumulative reward caused by stochastic
rewards.
VPAC-BS measures cross-replicate disagreement among Q-table
predictions, targeting epistemic uncertainty about the Q-function
arising from finite training data.
In the HTS environment the surrogate model is highly accurate
($R^2 = 0.9964$), so the dominant source of outcome variability is
aleatoric: sensor noise, thermal transients, and boundary-proximity
effects encoded in the noise model.
The VPAC-BS epistemic signal is correspondingly weak relative to this
aleatoric noise, and the estimator cannot reliably guide the actor
toward lower-variance behaviour at moderate $\beta$ values.
VPAC-RS is therefore the recommended estimator for environments in
which a reliable predictive model is available and aleatoric noise
dominates.
Conversely, VPAC-BS is better suited to settings with sparse data or
substantial model uncertainty, consistent with the competitive VPAC-BS
results in the puddle world experiments
(Section~\ref{sec:experiments}).
Section~\ref{sec:conclusion} discusses this estimator-selection
principle in the context of broader deployment guidance.

\subsection{Discussion}
\label{subsec:hts-discussion}

The HTS case study establishes three findings that go beyond the
puddle world experiments.

The first finding is that VPAC-RS scales to real industrial data without
modification.
The tile-coded function approximation, two-timescale actor-critic
updates, and rolling-sample variance estimator all operate directly on
real sensor features, without domain-specific reward shaping or noise
model knowledge.
The 74\% SS~IC~std reduction at $\beta = 0.01$ was achieved using the
same algorithm that reduced puddle world return variance by 83\% in the
tabular setting.

The second finding concerns the interpretation of the risk-return
trade-off.
The risk-neutral baseline ($\beta = 0$) achieves a higher mean $I_c$
(339.7~A) than the best VPAC-RS policy (318.3~A), confirming that
the variance reduction comes at a genuine cost in mean performance.
This is expected behaviour for mean-variance optimisation and should
not be interpreted as a failure of the algorithm.
Rather, it reflects the fundamental trade-off that the Pareto plot
in Figure~\ref{fig:hts_pareto} makes explicit: practitioners select
an operating point on the frontier based on their specific
yield/consistency requirements, and the appropriate $\beta$ will vary
across production lines and quality targets.

The third finding provides empirical support for a general principle
about estimator selection.
The VPAC-RS versus VPAC-BS comparison on HTS, taken together with the
competitive performance of both estimators on the puddle world, suggests
that the dominant uncertainty type in the environment determines which
estimator is more effective.
Aleatoric-dominated environments favour RS; epistemic-dominated
environments favour BS.
This principle is stated more precisely as deployment guidance in
Section~\ref{sec:conclusion}.

\section{Conclusion}
\label{sec:conclusion}

This paper introduced a nonparametric variance-penalized framework
for risk-sensitive reinforcement learning that eliminates the variance
critic of prior dual-critic methods by substituting statistically
grounded online estimators based on random scaling (VPAC-RS) and online
bootstrapping (VPAC-BS).
Almost-sure convergence of both the Q-learning and actor-critic variants
was established via the ODE method, requiring only that the variance
penalty remain bounded, a strictly weaker condition than the consistency
assumed by prior work.
Empirically, VPAC-RS achieves up to 83\% reduction in steady-state
return variance in the puddle world benchmark, matching or exceeding the
dual-critic VPAC baseline while adding zero learned parameters.
In the HTS manufacturing case study, the same algorithm reduces
steady-state $I_c$ variability by 74\% and episode return standard
deviation by 63\%, with the risk-return trade-off characterised
explicitly by a Pareto analysis across penalty coefficients.

A practical finding across environments is that estimator choice should
reflect the dominant uncertainty type: VPAC-RS is better matched to
environments with predominantly aleatoric noise, while VPAC-BS is
better suited to settings where epistemic uncertainty in the value
function is the primary source of policy instability.

Several directions remain open. Deriving finite-sample rates requires combining SA concentration
bounds~\cite{xu2020improving, wu2020finite} with finite-sample
approximation bounds for the present framework, a combination that existing analyses do not address and that
requires tools from the functional CLT
literature~\cite{xie2022statistical, li2023online}.
Extending the framework to neural network function approximation,
where bootstrap ensembles and running parameter-space statistics
provide natural analogues of the tabular estimators, is a practical
direction with direct implications for deep RL deployment in
safety-critical settings.


\bibliographystyle{IEEEtran}
\bibliography{refs_v2_1}

@article{garcia2015comprehensive,
  title={A Comprehensive Survey on Safe Reinforcement Learning},
  author={Garc{\'i}a, Javier and Fern{\'a}ndez, Fernando},
  journal={Journal of Machine Learning Research},
  volume={16},
  number={1},
  pages={1437--1480},
  year={2015}
}

@inproceedings{tamar2012policy,
  title={Policy Gradient with Variance Related Risk Criteria},
  author={Tamar, Aviv and Di Castro, Dotan and Mannor, Shie},
  booktitle={International Conference on Machine Learning},
  pages={935--942},
  year={2012}
}

@article{zhong2023variance,
  title={Risk-Sensitive Deep RL: Variance-Constrained Actor-Critic Provably Finds Globally Optimal Policy},
  author={Zhong, Han and Wu, Yuejie Chi and Wang, Zhaoran},
  journal={arXiv preprint arXiv:2012.14098},
  year={2023}
}

@article{chen2020statistical,
  title={Statistical inference for model parameters in stochastic gradient descent},
  author={Chen, Xi and Lee, Jason D and Tong, Xin T and Zhang, Yichen},
  year={2020}
}

@book{borkar2008stochastic,
  title={Stochastic Approximation: A Dynamical Systems Viewpoint},
  author={Borkar, Vivek S},
  year={2008},
  publisher={Springer}
}

@inproceedings{zhu2021finite,
  title={On constructing confidence region for model parameters in stochastic gradient descent via batch means},
  author={Zhu, Yuantao and Dong, Jingyi},
  booktitle={2021 Winter Simulation Conference (WSC)},
  pages={1--12},
  year={2021},
  organization={IEEE}
}

@article{Shen_2014,
   title={Risk-Sensitive Reinforcement Learning},
   volume={26},
   ISSN={1530-888X},
   url={http://dx.doi.org/10.1162/NECO_a_00600},
   DOI={10.1162/neco_a_00600},
   number={7},
   journal={Neural Computation},
   publisher={MIT Press - Journals},
   author={Shen, Yun and Tobia, Michael J. and Sommer, Tobias and Obermayer, Klaus},
   year={2014},
   month=jul, pages={1298–1328} }

@article{ramprasad2023online,
  title={Online bootstrap inference for policy evaluation in reinforcement learning},
  author={Ramprasad, Pratik and Li, Yuantong and Yang, Zhuoran and Wang, Zhaoran and Sun, Will Wei and Cheng, Guang},
  journal={Journal of the American Statistical Association},
  volume={118},
  number={544},
  pages={2901--2914},
  year={2023},
  publisher={Taylor \& Francis}
}

@article{xie2022statistical,
  title={A statistical online inference approach in averaged stochastic approximation},
  author={Xie, Chuhan and Zhang, Zhihua},
  journal={Advances in Neural Information Processing Systems},
  volume={35},
  pages={8998--9009},
  year={2022}
}

@article{chow2018risk,
  title={Risk-constrained reinforcement learning with percentile risk criteria},
  author={Chow, Yinlam and Ghavamzadeh, Mohammad and Janson, Lucas and Pavone, Marco},
  journal={Journal of Machine Learning Research},
  volume={18},
  number={167},
  pages={1--51},
  year={2018}
}

@article{sobel1982variance,
  title={The variance of discounted Markov decision processes},
  author={Sobel, Matthew J},
  journal={Journal of Applied Probability},
  volume={19},
  number={4},
  pages={794--802},
  year={1982},
  publisher={Cambridge University Press}
}

@article{white1988mean,
author = {White, D. J.},
title = {Mean, variance, and probabilistic criteria in finite Markov decision processes: a review},
year = {1988},
issue_date = {January 1988},
publisher = {Plenum Press},
address = {USA},
volume = {56},
number = {1},
issn = {0022-3239},
url = {https://doi.org/10.1007/BF00938524},
doi = {10.1007/BF00938524},
journal = {J. Optim. Theory Appl.},
month = jan,
pages = {1–29},
numpages = {29}
}

@article{tamar2013variance,
  title={Variance adjusted actor critic algorithms},
  author={Tamar, Aviv and Mannor, Shie},
  journal={arXiv preprint arXiv:1310.3697},
  year={2013}
}

@article{la2013actor,
  title={Actor-critic algorithms for risk-sensitive MDPs},
  author={La, Prashanth and Ghavamzadeh, Mohammad},
  journal={Advances in neural information processing systems},
  volume={26},
  year={2013}
}

@inproceedings{tamar2012variance,
author = {Tamar, Aviv and Di Castro, Dotan and Mannor, Shie},
title = {Policy gradients with variance related risk criteria},
year = {2012},
isbn = {9781450312851},
publisher = {Omnipress},
address = {Madison, WI, USA},
booktitle = {Proceedings of the 29th International Coference on International Conference on Machine Learning},
pages = {1651–1658},
numpages = {8},
location = {Edinburgh, Scotland},
series = {ICML'12}
}

@article{prashanth2016variance,
  title={Variance-constrained actor-critic algorithms for discounted and average reward MDPs},
  author={Prashanth, LA and Ghavamzadeh, Mohammad},
  journal={Machine Learning},
  volume={105},
  pages={367--417},
  year={2016},
  publisher={Springer}
}

@article{li2023online,
  title={Online statistical inference for nonlinear stochastic approximation with markovian data},
  author={Li, Xiang and Liang, Jiadong and Zhang, Zhihua},
  journal={arXiv preprint arXiv:2302.07690},
  year={2023}
}

@inproceedings{jain2021variance,
  title={Variance penalized on-policy and off-policy actor-critic},
  author={Jain, Arushi and Patil, Gandharv and Jain, Ayush and Khetarpal, Khimya and Precup, Doina},
  booktitle={Proceedings of the AAAI Conference on Artificial Intelligence},
  volume={35},
  number={9},
  pages={7899--7907},
  year={2021}
}

@article{mnih2015human,
  title={Human-level control through deep reinforcement learning},
  author={Mnih, Volodymyr and Kavukcuoglu, Koray and Silver, David and Rusu, Andrei A and Veness, Joel and Bellemare, Marc G and Graves, Alex and Riedmiller, Martin and Fidjeland, Andreas K and Ostrovski, Georg and others},
  journal={Nature},
  volume={518},
  number={7540},
  pages={529--533},
  year={2015},
  publisher={Nature Publishing Group},
  doi={10.1038/nature14236}
}

@article{silver2017mastering,
  title={Mastering the game of {Go} without human knowledge},
  author={Silver, David and Schrittwieser, Julian and Simonyan, Karen and Antonoglou, Ioannis and Huang, Aja and Guez, Arthur and Hubert, Thomas and Baker, Lucas and Lai, Matthew and Bolton, Adrian and others},
  journal={Nature},
  volume={550},
  number={7676},
  pages={354--359},
  year={2017},
  publisher={Nature Publishing Group},
  doi={10.1038/nature24270}
}

@article{levine2016end,
  title={End-to-end training of deep visuomotor policies},
  author={Levine, Sergey and Finn, Chelsea and Darrell, Trevor and Abbeel, Pieter},
  journal={Journal of Machine Learning Research},
  volume={17},
  number={39},
  pages={1--40},
  year={2016}
}

@inproceedings{kuderer2015learning,
  title={Learning driving styles for autonomous vehicles from demonstration},
  author={Kuderer, Markus and Gulati, Shilpa and Burgard, Wolfram},
  booktitle={IEEE International Conference on Robotics and Automation (ICRA)},
  pages={2641--2646},
  year={2015},
  organization={IEEE},
  doi={10.1109/ICRA.2015.7139555}
}

@article{komorowski2018artificial,
  title={The artificial intelligence clinician learns optimal treatment strategies for sepsis in intensive care},
  author={Komorowski, Matthieu and Celi, Leo A and Badawi, Omar and Gordon, Anthony C and Faisal, A Aldo},
  journal={Nature Medicine},
  volume={24},
  number={11},
  pages={1716--1720},
  year={2018},
  publisher={Nature Publishing Group},
  doi={10.1038/s41591-018-0213-5}
}

@article{markowitz1952portfolio,
  title={Portfolio selection},
  author={Markowitz, Harry},
  journal={The Journal of Finance},
  volume={7},
  number={1},
  pages={77--91},
  year={1952},
  publisher={Wiley}
}

@article{howard1972risk,
  title={Risk-sensitive {Markov} decision processes},
  author={Howard, Ronald A and Matheson, James E},
  journal={Management Science},
  volume={18},
  number={7},
  pages={356--369},
  year={1972},
  publisher={INFORMS}
}

@article{borkar2002q,
  title={{Q-Learning} for risk-sensitive control},
  author={Borkar, Vivek S},
  journal={Mathematics of Operations Research},
  volume={27},
  number={2},
  pages={294--311},
  year={2002},
  publisher={INFORMS},
  doi={10.1287/moor.27.2.294.324}
}

@inproceedings{borkar2010learning,
  title={Learning Algorithms for Risk-Sensitive Control},
  author={Vivek S. Borkar},
  year={2010},
  url={https://api.semanticscholar.org/CorpusID:5835100}
}

@article{al-dabooni2019boundedness,
  title={The boundedness conditions for model-free {HDP}($\lambda$)},
  author={Al-Dabooni, Seaar and Wunsch, Donald},
  journal={IEEE Transactions on Neural Networks and Learning Systems},
  volume={30},
  number={7},
  pages={1928--1942},
  year={2019},
  publisher={IEEE},
  doi={10.1109/TNNLS.2018.2875870}
}

@inproceedings{wu2020finite,
  title={A finite-time analysis of two time-scale actor-critic methods},
  author={Wu, Yue and Zhang, Weitong and Xu, Pan and Gu, Quanquan},
  booktitle={Advances in Neural Information Processing Systems (NeurIPS)},
  volume={33},
  pages={17617--17628},
  year={2020}
}

@inproceedings{xu2020improving,
  title={Improving sample complexity bounds for (natural) actor-critic algorithms},
  author={Xu, Tengyu and Wang, Zhe and Liang, Yingbin},
  booktitle={Advances in Neural Information Processing Systems (NeurIPS)},
  volume={33},
  pages={4358--4369},
  year={2020}
}

@inproceedings{sherstan2018comparing,
  title={Comparing direct and indirect temporal-difference methods for estimating the variance of the return},
  author={Sherstan, Craig and Ashley, Dylan R and Bennett, Brendan and Young, Kenny and White, Adam and White, Martha and Sutton, Richard S},
  booktitle={Conference on Uncertainty in Artificial Intelligence (UAI)},
  year={2018}
}

@inproceedings{bisi2020risk,
  title={Risk-averse trust region optimization for reward-volatility reduction},
  author={Bisi, Lorenzo and Sabbioni, Luca and Vittori, Edoardo and Papini, Matteo and Restelli, Marcello},
  booktitle={Proceedings of the Twenty-Ninth International Joint Conference on Artificial Intelligence (IJCAI)},
  pages={4583--4589},
  year={2020},
  doi={10.24963/ijcai.2020/632}
}

@inproceedings{wei2017deep,
  title={Deep reinforcement learning for building {HVAC} control},
  author={Wei, Tianshu and Wang, Yanzhi and Zhu, Qi},
  booktitle={Proceedings of the 54th Annual Design Automation Conference (DAC)},
  pages={1--6},
  year={2017},
  doi={10.1145/3061639.3062224}
}

@article{nian2020review,
  title={A review on reinforcement learning: Introduction and applications in industrial process control},
  author={Nian, Rui and Liu, Jinfeng and Huang, Biao},
  journal={Computers \& Chemical Engineering},
  volume={139},
  pages={106886},
  year={2020},
  publisher={Elsevier},
  doi={10.1016/j.compchemeng.2020.106886}
}

@inproceedings{tamar2015optimizing,
  title={Optimizing the {CVaR} via sampling},
  author={Tamar, Aviv and Glassner, Yonatan and Mannor, Shie},
  booktitle={Proceedings of the AAAI Conference on Artificial Intelligence},
  volume={29},
  number={1},
  year={2015},
  doi={10.1609/aaai.v29i1.9561}
}

@book{liu2017adaptive,
  title={Adaptive Dynamic Programming with Applications in Optimal Control},
  author={Liu, Derong and Wei, Qinglai and Wang, Ding and Yang, Xiong and Li, Hongliang},
  series={Advances in Industrial Control},
  year={2017},
  publisher={Springer},
  address={Cham, Switzerland},
  doi={10.1007/978-3-319-50815-3}
}

@article{ma2023distributional,
  title={Distributional reinforcement learning for run-to-run control in semiconductor manufacturing processes},
  author={Zhu, Ma and Pan, Tianhong},
  journal={Neural Computing and Applications},
  volume={35},
  number={26},
  pages={19337--19350},
  year={2023},
  publisher={Springer},
  doi={10.1007/s00521-023-08760-1}
}

@article{bhatnagar2009natural,
  title={Natural actor-critic algorithms},
  author={Bhatnagar, Shalabh and Sutton, Richard S and Ghavamzadeh, Mohammad and Lee, Mark},
  journal={Automatica},
  volume={45},
  number={11},
  pages={2471--2482},
  year={2009},
  publisher={Elsevier},
  doi={10.1016/j.automatica.2009.07.008}
}

@article{konda2003onactor,
  title={On actor-critic algorithms},
  author={Konda, Vijay R and Tsitsiklis, John N},
  journal={SIAM Journal on Control and Optimization},
  volume={42},
  number={4},
  pages={1143--1166},
  year={2003},
  publisher={SIAM}
}

@inproceedings{bellemare2017distributional,
  title={A distributional perspective on reinforcement learning},
  author={Bellemare, Marc G and Dabney, Will and Munos, R{\'e}mi},
  booktitle={International Conference on Machine Learning (ICML)},
  pages={449--458},
  year={2017},
  organization={PMLR}
}

@inproceedings{dabney2018distributional,
  title={Distributional reinforcement learning with quantile regression},
  author={Dabney, Will and Rowland, Mark and Bellemare, Marc G and Munos, R{\'e}mi},
  booktitle={AAAI Conference on Artificial Intelligence},
  volume={32},
  number={1},
  year={2018}
}

@inproceedings{dabney2018implicit,
  title={Implicit quantile networks for distributional reinforcement learning},
  author={Dabney, Will and Ostrovski, Georg and Silver, David and Munos, R{\'e}mi},
  booktitle={International Conference on Machine Learning (ICML)},
  pages={1096--1105},
  year={2018},
  organization={PMLR}
}

@article{sangadi2024finite,
  title={A finite-sample analysis of an actor-critic algorithm for mean-variance optimization in a discounted {MDP}},
  author={Sangadi, Tejaram and Prashanth, L A and Jagannathan, Krishna},
  journal={arXiv preprint arXiv:2406.07892},
  year={2024}
}

@article{jaakkola1994convergence,
  title={Convergence of stochastic iterative dynamic programming algorithms},
  author={Jaakkola, Tommi and Jordan, Michael I and Singh, Satinder P},
  journal={Neural Computation},
  volume={6},
  number={6},
  pages={1185--1201},
  year={1994}
}

@article{tsitsiklis1994asynchronous,
  title={Asynchronous stochastic approximation and Q-learning},
  author={Tsitsiklis, John N},
  journal={Machine Learning},
  volume={16},
  number={3},
  pages={185--202},
  year={1994}
}

@article{tsitsiklis1997analysis,
  title={An analysis of temporal-difference learning with function approximation},
  author={Tsitsiklis, John N. and Van Roy, Benjamin},
  journal={IEEE Transactions on Automatic Control},
  volume={42},
  number={5},
  pages={674--690},
  year={1997},
  publisher={IEEE}
}

@article{panda2025asymptotic,
  author  = {Panda, Saunak Kumar and Liu, Ruiqi and Xiang, Yisha},
  title   = {Asymptotic Analysis of Sample-averaged {Q}-learning},
  journal = {IEEE Transactions on Information Theory},
  year    = {2025},
  doi     = {10.1109/TIT.2025.3569421}
}

\newpage
 
 \appendix
\section{Hyperparameter Settings}
\label{app:hyperparams}
 
Table~\ref{tab:hyperparams} lists the full hyperparameter settings
used across all puddle world experiments.
Summary statistics reported in the main paper are computed from
unsmoothed rollouts of the converged policy; smoothing coefficients
apply only to learning curve figures.
 
\begin{table}[htbp]
\centering
\caption{Hyperparameters used across discrete and continuous puddle
  world settings.
  $lr_c$: critic learning rate.
  $lr_a$: actor learning rate.
  $lr_\sigma$: variance critic learning rate (VPAC only).}
\label{tab:hyperparams}
\begin{tabular}{lll}
\toprule
\textbf{Parameter} & \textbf{Discrete} & \textbf{Continuous} \\
\midrule
\multicolumn{3}{l}{\textit{Common parameters}} \\
Discount $\gamma$ & 0.99 & 0.99 \\
Max episodes & 1{,}000 & 2{,}000 \\
Max steps & 500 & 5{,}000 \\
Seeds & 10 & 10 \\
Smoothing ($\alpha_\mathrm{smooth}$, returns) & 0.03 & 0.03 \\
Smoothing ($\alpha_\mathrm{smooth}$, variance) & 0.05 & 0.05 \\
\midrule
\multicolumn{3}{l}{\textit{AC (baseline, $\beta = 0$)}} \\
$lr_c$ / $lr_a$ & 0.1 / 0.01 & 0.05 / 0.01 \\
\midrule
\multicolumn{3}{l}{\textit{VPAC~\cite{jain2021variance}}} \\
$\beta$ & 0.002 & 0.005 \\
$lr_c$ / $lr_a$ / $lr_\sigma$ & 0.1 / 0.01 / 0.1 & 0.05 / 0.01 / 0.025 \\
\midrule
\multicolumn{3}{l}{\textit{VPAC-RS (proposed)}} \\
$\beta$ & 0.01 & 0.005 \\
$lr_c$ / $lr_a$ & 0.1 / 0.01 & 0.05 / 0.01 \\
Variance update frequency & 20 steps & 20 steps \\
Warm-up episodes & 5 & 5 \\
\midrule
\multicolumn{3}{l}{\textit{VPAC-BS (proposed)}} \\
$\beta$ & 0.005 & 0.005 \\
$lr_c$ / $lr_a$ & 0.1 / 0.01 & 0.05 / 0.01 \\
Ensemble size $K$ & 10 & 10 \\
Variance update frequency & 20 steps & 20 steps \\
\midrule
\multicolumn{3}{l}{\textit{Continuous setting only}} \\
Tile coding & -- & 10 tilings, 5 bins \\
Feature dimension & -- & 1{,}024 \\
\bottomrule
\end{tabular}
\end{table}

\section{Proof of Theorem~\ref{thm:qconv}}
\label{app:proof-qconv}




We prove almost-sure convergence of the variance-penalized Q-learning
iterates to $Q^*$, the unique fixed point of the penalized Bellman
operator $K$, by verifying the stochastic approximation conditions of
Borkar's Theorem~2.2~\cite{borkar2008stochastic}.
Since the update modifies only the visited state-action pair at each
iteration, the recursion is an asynchronous stochastic approximation;
we invoke the standard reduction of asynchronous updates with positive
asymptotic visitation frequency to the synchronous framework
(Jaakkola et~al.~\cite{jaakkola1994convergence}; see also
Tsitsiklis~\cite{tsitsiklis1994asynchronous} and Chapter~7 of
Borkar~\cite{borkar2008stochastic}).

\subsubsection*{Step 1: SA Recursion in Standard Form}

We write the Q-learning update as
\begin{equation}
\begin{aligned}
Q_{t+1}(s,a)
&= Q_t(s,a)
+ \Delta_t(s,a), \\
\Delta_t(s,a)
&= \alpha_t\,
\mathbf{1}\!\left[(s,a)=(S_t,A_t)\right] \\
&\quad\times
\bigl[h(Q_t)(s,a)
+ M_{t+1}(s,a)\bigr],
\end{aligned}
\label{eq:sa-ql}
\end{equation}
where the mean-field function is
\begin{equation}
h(Q)(s,a) = (KQ)(s,a) - Q(s,a),
\end{equation}
and the noise term at the visited state-action pair is
\begin{equation}
M_{t+1}(s,a) =
\begin{cases}
\begin{aligned}
& R_{t+1}
+ \gamma \max_\pi F_{Q_t}^\pi(S_{t+1}) \\
&\quad - (KQ_t)(s,a),
\end{aligned}
& (s,a) = (S_t,A_t), \\[1ex]
0, & \text{otherwise.}
\end{cases}
\label{eq:noise-ql}
\end{equation}
Under Assumption~\ref{ass:ergodic}, the fixed exploration policy
$\pi_b$ induces a stationary distribution $\mu^{\pi_b}$ with
$\mu^{\pi_b}(s,a) > 0$ for every $(s,a)$, so that
\begin{equation}
\frac{1}{t}\sum_{k=0}^{t-1}
  \mathbf{1}\!\left[(S_k,A_k)=(s,a)\right]
  \;\longrightarrow\; \mu^{\pi_b}(s,a) > 0
\quad \text{a.s.}
\end{equation}
Each state-action pair is therefore updated with positive asymptotic
frequency, which is the condition required for the induced
per-coordinate asynchronous step sizes to inherit the Robbins-Monro
conditions from $\{\alpha_t\}$.

\subsubsection*{Step 2: $\gamma$-Contraction of $K$}

We show $\|KQ_1 - KQ_2\|_\infty \leq \gamma\|Q_1 - Q_2\|_\infty$ for
any $Q_1, Q_2$.
For any state $s'$, define for $i = 1, 2$:
\begin{equation}
F_{Q_i}^\pi(s') = \sum_{a'} \pi(a'|s')\bigl[Q_i(s',a')
  - \beta\,\hat{\sigma}_\mathrm{clip}(s',a')\bigr],
\end{equation}
where $\hat{\sigma}_\mathrm{clip}$ is treated as a fixed bounded
state-action penalty inside the Bellman operator $K$ for the purposes
of this contraction argument, and is the same in both $F_{Q_1}^\pi$
and $F_{Q_2}^\pi$.
Let $\pi_1^* = \arg\max_\pi F_{Q_1}^\pi(s')$.
Since $\pi_2^* = \arg\max_\pi F_{Q_2}^\pi(s')$ maximises
$F_{Q_2}^\pi$, we have $F_{Q_2}^{\pi_2^*}(s') \geq F_{Q_2}^{\pi_1^*}(s')$,
and therefore:
\begin{align}
\max_\pi F_{Q_1}^\pi(s')
- \max_\pi F_{Q_2}^\pi(s')
&=
F_{Q_1}^{\pi_1^*}(s')
-
F_{Q_2}^{\pi_2^*}(s')
\nonumber\\
&\le
F_{Q_1}^{\pi_1^*}(s')
-
F_{Q_2}^{\pi_1^*}(s')
\nonumber\\
&=
\sum_{a'}\pi_1^*(a'|s')
\nonumber\\
&\qquad\times
\bigl(Q_1(s',a')-Q_2(s',a')\bigr)
\nonumber\\
&\le
\|Q_1-Q_2\|_\infty,
\label{eq:max-upper}
\end{align}
where the last equality uses the fact that $\hat{\sigma}_\mathrm{clip}$
cancels in the difference $F_{Q_1}^{\pi_1^*} - F_{Q_2}^{\pi_1^*}$,
since both terms use the same fixed penalty.
By the symmetric argument with $\pi_2^*$:
\begin{equation}
\max_\pi F_{Q_2}^\pi(s') - \max_\pi F_{Q_1}^\pi(s')
  \leq \|Q_1 - Q_2\|_\infty.
\label{eq:max-lower}
\end{equation}
Combining~\eqref{eq:max-upper} and~\eqref{eq:max-lower}:
\begin{equation}
\bigl|\max_\pi F_{Q_1}^\pi(s') - \max_\pi F_{Q_2}^\pi(s')\bigr|
  \leq \|Q_1 - Q_2\|_\infty.
\end{equation}
Taking expectations over $S'$ conditional on $(s,a)$ and then taking
the supremum over $(s,a)$ gives
\begin{equation}
\|(KQ_1) - (KQ_2)\|_\infty \leq \gamma \|Q_1 - Q_2\|_\infty.
\label{eq:contraction}
\end{equation}
Since $\gamma < 1$, $K$ is a $\gamma$-contraction in $\|\cdot\|_\infty$.
By Banach's fixed-point theorem, $K$ has a unique fixed point $Q^*$,
satisfying $\|Q^*\|_\infty \leq (R_\mathrm{max} + \gamma\beta
c_\mathrm{max})/(1-\gamma)$, and $h(Q) = KQ - Q$ satisfies
$\|h(Q_1) - h(Q_2)\|_\infty \leq (1+\gamma)\|Q_1-Q_2\|_\infty$
(Lipschitz condition, Borkar's Assumption~A1).



\subsubsection*{Step 3: Martingale Difference Noise}

For the visited coordinate $(S_t,A_t)$, the noise
in~\eqref{eq:noise-ql} satisfies
\begin{equation}
\mathbb{E}\!\left[M_{t+1}(S_t,A_t)\mid\mathcal{F}_t\right] = 0
\end{equation}
by definition of $(KQ_t)(S_t,A_t)$ as the conditional expectation of
$R_{t+1} + \gamma\max_\pi F_{Q_t}^\pi(S_{t+1})$ given $(S_t,A_t)$. For the conditional second moment, note that
$|R_{t+1}| \leq R_\mathrm{max}$ almost surely (Assumption~\ref{ass:rewards})
and
\begin{equation}
\bigl|\max_\pi F_{Q_t}^\pi(s')\bigr|
  \leq \|Q_t\|_\infty + \beta c_\mathrm{max}
\end{equation}
by Assumption~\ref{ass:bounded} and the fact that
$|\sum_{a'}\pi(a'|s')Q_t(s',a')| \leq \|Q_t\|_\infty$.
Therefore
\begin{equation}
\mathbb{E}\bigl[M_{t+1}^2(S_t,A_t)\mid\mathcal{F}_t\bigr]
  \leq C\bigl(1 + \|Q_t\|_\infty^2\bigr)
\end{equation}
for a constant $C > 0$ depending only on $R_\mathrm{max}$, $\gamma$,
$\beta$, and $c_\mathrm{max}$ (Borkar's Assumption~A3).



\subsubsection*{Step 4: Step-Size and Bounded-Iterate Conditions}

Assumption~\ref{ass:stepsizes} supplies the Robbins-Monro conditions
for the base step-size sequence $\{\alpha_t\}$, while Step~1 shows
that the induced per-coordinate asynchronous step sizes inherit these
conditions under positive asymptotic visitation frequency
(Borkar's Assumption~A2).
Assumption~\ref{ass:bounded-iterates} supplies almost-sure boundedness
of $\{Q_t\}$ (Borkar's Assumption~A4).



\subsubsection*{Step 5: Global Asymptotic Stability of the ODE}

The limiting ODE is $\dot{Q}(t) = h(Q(t)) = KQ(t) - Q(t)$.
Since $KQ^* = Q^*$, the error $e(t) = Q(t) - Q^*$ satisfies
$\dot{e}(t) = KQ(t) - KQ^* - e(t)$.
Let $V(t) = \|e(t)\|_\infty$ and define the active set
\begin{equation}
\mathcal{I}(t) = \bigl\{(s,a) : |e(t)(s,a)| = V(t)\bigr\}.
\end{equation}
For $V(t) > 0$, the upper right Dini derivative satisfies
\begin{align}
D^+ V(t)
&\leq \max_{(s,a)\in\mathcal{I}(t)}
  \operatorname{sgn}(e(t)(s,a))\,\dot{e}(t)(s,a)
\nonumber \\
&= \max_{(s,a)\in\mathcal{I}(t)}
  \operatorname{sgn}(e(t)(s,a))
  \Bigl[
    (KQ(t))(s,a) - (KQ^*)(s,a)
    - e(t)(s,a)
  \Bigr]
\nonumber \\
&\leq \max_{(s,a)\in\mathcal{I}(t)}
  \bigl|(KQ(t))(s,a) - (KQ^*)(s,a)\bigr|
  - V(t)
\nonumber \\
&\leq \gamma\|Q(t)-Q^*\|_\infty - V(t)
\nonumber \\
&= -(1-\gamma)V(t).
\label{eq:gronwall}
\end{align}
When $V(t) = 0$, the inequality holds trivially.
Therefore, by the comparison principle for upper right Dini
derivatives,
\begin{equation}
V(t) \leq e^{-(1-\gamma)t}V(0),
\end{equation}
so $Q(t) \to Q^*$ exponentially as $t \to \infty$, establishing global
asymptotic stability of the limiting ODE (Borkar's Assumption~A5).

All conditions of Borkar's Theorem~2.2 are satisfied.
Therefore $Q_t \to Q^*$ almost surely. \hfill$\square$

\begin{remark}
The contraction argument in Step~2 holds for any bounded
$\hat{\sigma}_\mathrm{clip}$, regardless of whether it converges to
the true $\sigma^\pi$.
The key algebraic fact is that $\hat{\sigma}_\mathrm{clip}$ cancels
in $F_{Q_1}^{\pi_1^*} - F_{Q_2}^{\pi_1^*}$, since both use the same
estimator.
This is why boundedness is the only requirement on the variance
estimate: the fixed point $Q^*$ exists and is unique for any
bounded $\hat{\sigma}_\mathrm{clip}$, and the operator $K$ is a
$\gamma$-contraction irrespective of the estimator's accuracy.
\end{remark}

\section{Proof of Theorem~\ref{thm:acconv}}
\label{app:proof-acconv}

We prove convergence of the two-timescale actor-critic using Borkar's
Theorem~2 of Chapter~6~\cite{borkar2008stochastic}, which gives
conditions under which a two-timescale SA scheme converges: the
faster-timescale iterate tracks a unique stable equilibrium that
depends continuously on the slower-timescale iterate, and the
slower-timescale iterate then has its limit set contained in the
stationary set of the ODE induced by evaluating the equilibrium map
at the current slower iterate.

The on-policy case ($\rho_t \equiv 1$) is treated here; the
off-policy extension is given in Appendix~\ref{app:proof-offpolicy}.
Throughout, we assume bounded features $\|\phi(s,a)\| \leq
\phi_\mathrm{max}$ and a bounded, Lipschitz policy score on the
relevant compact parameter set, in addition to
Assumptions~\ref{ass:ergodic}--\ref{ass:smooth}.

\subsubsection*{Step 1: Two-Timescale SA Form}

Write the critic update as
\begin{equation}
w_{t+1} = w_t + \alpha_{w,t}\bigl[h_w(w_t, \theta_t) + M_{t+1}^w\bigr],
\label{eq:critic-sa}
\end{equation}
where the critic mean field is
\begin{equation}
\begin{aligned}
h_w(w, \theta)
&= \mathbb{E}_{\mu^{\pi_\theta}}\!\Bigl[
    \bigl(R_{t+1} + \gamma F(S_{t+1})
    - Q_w(S_t,A_t)\bigr) \\
&\qquad\qquad \times
    \nabla_w Q_w(S_t,A_t)
  \Bigr].
\end{aligned}
\end{equation}
with $F(s') = \sum_{a'}\pi_\theta(a'|s')[Q_w(s',a')
  - \beta\hat{\sigma}_\mathrm{clip}(s',a')]$,
and the critic noise, centered at its stationary mean field, is
\begin{equation}
\begin{aligned}
M_{t+1}^w
&= \bigl(R_{t+1} + \gamma F(S_{t+1})
    - Q_w(S_t,A_t)\bigr) \nabla_w Q_w(S_t,A_t)\\
&\qquad 
    - h_w(w_t,\theta_t).
\end{aligned}
\label{eq:critic-noise}
\end{equation}
Write the actor update as
\begin{equation}
\theta_{t+1} = \theta_t + \alpha_{\theta,t}\bigl[g(\theta_t, w_t)
  + M_{t+1}^\theta\bigr],
\label{eq:actor-sa}
\end{equation}
where the actor mean field is
\begin{equation}
\begin{aligned}
g(\theta, w)
&= \mathbb{E}_{\mu^{\pi_\theta}}\!\bigl[
      \hat{Q}_\mathrm{var}(S_t,A_t)\,
      \nabla_\theta\log\pi_\theta(A_t|S_t)
    \bigr], \\
\hat{Q}_\mathrm{var}(s,a)
&= Q_w(s,a)
  - \beta\hat{\sigma}_\mathrm{clip}(s,a).
\end{aligned}
\end{equation}
and the actor noise, centered at its stationary mean field, is
\begin{equation}
M_{t+1}^\theta
  = \hat{Q}_\mathrm{var}(S_t,A_t)\nabla_\theta\log\pi_\theta(A_t|S_t)
    - g(\theta_t, w_t).
\label{eq:actor-noise}
\end{equation}
Because $(S_t, A_t)$ is generated by a Markov chain rather than
i.i.d.\ sampling from $\mu^{\pi_\theta}$ at every step, $M_{t+1}^w$
and $M_{t+1}^\theta$ are not literal martingale differences with
respect to the full filtration $\mathcal{F}_t$; instead they satisfy
the bounded controlled Markov noise conditions of Borkar's
two-timescale framework (Chapter~6,
Section~6.2 of~\cite{borkar2008stochastic}), which requires only
that the noise, averaged over the mixing time of the chain, vanish
asymptotically relative to the step size, together with the
conditional second-moment growth bound verified in Step~2 and Step~4
below. This is the standard treatment for actor-critic recursions
built on a single Markovian trajectory~\cite{bhatnagar2009natural,
konda2003onactor}.

\subsubsection*{Step 2: Faster-Timescale Analysis
  (Critic Convergence to $w^*(\theta)$)}

For a quasi-static actor parameter $\theta$, the critic recursion
\eqref{eq:critic-sa} tracks the ODE
\begin{equation}
\dot{w}(t) = h_w(w(t), \theta).
\label{eq:critic-ode}
\end{equation}

\textit{Linear FA mean field.}
Under linear function approximation $Q_w(s,a) = w^\top\phi(s,a)$,
the TD target becomes
$y_t = R_{t+1} + \gamma(w^\top\Phi(S_{t+1},\pi_\theta)
  - \beta B(S_{t+1},\pi_\theta))$,
where $\Phi(s',\pi) = \sum_{a'}\pi(a'|s')\phi(s',a')$ and
$B(s',\pi) = \sum_{a'}\pi(a'|s')\hat{\sigma}_\mathrm{clip}(s',a')$.
The mean field is
\begin{equation}
h_w(w, \theta) = b_\mathrm{var}(\theta) - G(\theta)w,
\qquad G(\theta) = A(\theta) - \gamma C(\theta),
\label{eq:critic-meanfield}
\end{equation}
where
\begin{align}
A(\theta) &= \mathbb{E}_{\mu^{\pi_\theta}}[\phi(s,a)\phi(s,a)^\top], \\
C(\theta) &= \mathbb{E}_{\mu^{\pi_\theta}}[\phi(s,a)\Phi(S',\pi_\theta)^\top], \\
b_\mathrm{var}(\theta) &= \mathbb{E}_{\mu^{\pi_\theta}}
  \bigl[(R - \gamma\beta B(S',\pi_\theta))\phi(s,a)\bigr].
\end{align}

\textit{Positive definiteness.}
$G(\theta) = A(\theta) - \gamma C(\theta)$ is identical in structure to
the standard on-policy linear TD system matrix, for which positive
definiteness under Assumption~\ref{ass:ergodic} and non-degenerate
features is established in Tsitsiklis and
Van Roy~\cite{tsitsiklis1997analysis}.
The variance penalty modifies only $b_\mathrm{var}$ (a vector), not
$G(\theta)$, so positive definiteness is preserved.
We assume $G(\theta)$ is uniformly positive definite over the compact
parameter set, consistent with the standard linear TD stability
condition.

\textit{Unique stable equilibrium.}
Since $G(\theta)$ is positive definite, the
ODE~\eqref{eq:critic-ode} is globally exponentially stable with
unique equilibrium
\begin{equation}
w^*(\theta) = G(\theta)^{-1}b_\mathrm{var}(\theta).
\label{eq:wstar}
\end{equation}
The noise $M_{t+1}^w$ in~\eqref{eq:critic-noise} has zero mean under
$\mu^{\pi_\theta}$ and bounded conditional second moment, since
$R_{t+1}$, $\nabla_w Q_w = \phi$, $Q_w$, and
$\hat{\sigma}_\mathrm{clip}$ are all bounded under
Assumptions~\ref{ass:rewards},~\ref{ass:bounded},
and~\ref{ass:bounded-iterates}, together with the assumed feature
boundedness.
By Borkar's Theorem~2.2 (extended to controlled Markov noise as
discussed in Step~1), $w_t \to w^*(\theta)$ almost surely for each
fixed $\theta$.

\subsubsection*{Step 3: Continuity of $\theta \mapsto w^*(\theta)$}

From~\eqref{eq:wstar}, $w^*(\theta)$ is continuous in $\theta$
provided $A(\theta)$, $C(\theta)$, and $b_\mathrm{var}(\theta)$ are
continuous in $\theta$ and $G(\theta)$ is uniformly invertible.
The matrices $A$ and $C$ are computed as expectations under the
stationary distribution $\mu^{\pi_\theta}$.
Under Assumption~\ref{ass:smooth}, $\pi_\theta(a|s)$ is differentiable
(hence continuous) in $\theta$ for all $(s,a)$, and by the ergodicity
of the Markov chain (Assumption~\ref{ass:ergodic}), the stationary
distribution $\mu^{\pi_\theta}$ varies continuously with $\theta$
under smooth policy parameterizations (see, e.g., Section~2 of
Bhatnagar et al.~\cite{bhatnagar2009natural}).
Continuity of $A(\theta)$, $C(\theta)$, and $b_\mathrm{var}(\theta)$
follows.
Uniform invertibility of $G(\theta)$ holds because positive
definiteness is preserved uniformly over compact parameter sets under
non-degenerate, bounded features.
Hence $\theta \mapsto w^*(\theta)$ is continuous.

\subsubsection*{Step 4: Slower-Timescale Analysis (Actor Convergence)}

On the actor's timescale, the critic is at equilibrium $w^*(\theta_t)$
(Step~2).
The actor recursion~\eqref{eq:actor-sa} therefore approximately tracks
the ODE
\begin{equation}
\begin{aligned}
\dot{\theta}(t)
&= g\bigl(\theta(t), w^*(\theta(t))\bigr) \\
&= \mathbb{E}_{\mu^{\pi_{\theta(t)}}}\!\Bigl[
      \hat{Q}_\mathrm{var}(s,a;\theta(t)) \nabla_\theta\log\pi_{\theta(t)}(a|s)
    \Bigr].
\end{aligned}
\label{eq:actor-ode}
\end{equation}
where $\hat{Q}_\mathrm{var}(s,a;\theta) = Q_{w^*(\theta)}(s,a)
- \beta\hat{\sigma}_\mathrm{clip}(s,a)$, with
$\hat{\sigma}_\mathrm{clip}$ treated as a fixed bounded penalty within
this mean-field evaluation, consistent with the quasi-static
convention used throughout this proof.

By the policy gradient theorem applied to the variance-penalized
objective $J_\mathrm{var}(\theta)$, the right-hand side
of~\eqref{eq:actor-ode} equals $\nabla_\theta J_\mathrm{var}(\theta)$.
Thus the actor ODE is the gradient ascent flow of
$J_\mathrm{var}(\theta)$, and its equilibria satisfy
$\nabla_\theta J_\mathrm{var}(\theta^*) = 0$.

\textit{Lipschitz drift.}
The drift $g(\theta, w^*(\theta))$ is Lipschitz in $\theta$ on the
compact parameter set:
the score function $\nabla_\theta\log\pi_\theta(a|s)$ is bounded and
Lipschitz by Assumption~\ref{ass:smooth} and the standing boundedness
assumption of this appendix;
$\hat{Q}_\mathrm{var}$ is bounded by
$\|Q_{w^*(\theta)}\|_\infty + \beta c_\mathrm{max}$, which is bounded
over compact $\theta$ sets given bounded features and the continuity
of $w^*(\theta)$;
and $w^*(\theta)$ is Lipschitz in $\theta$ (it is continuously
differentiable by the implicit function theorem applied
to~\eqref{eq:wstar}).

\textit{Noise conditions.}
The actor noise $M_{t+1}^\theta$ in~\eqref{eq:actor-noise} is centered
at its mean field under $\mu^{\pi_{\theta_t}}$, and by the same
controlled Markov noise argument as Step~1, satisfies the conditional
second-moment bound
$\mathbb{E}[\|M_{t+1}^\theta\|^2\mid\mathcal{F}_t]
  \leq C'(1 + \|\theta_t\|^2)$,
because the score function and $\hat{Q}_\mathrm{var}$ are bounded
under Assumptions~\ref{ass:bounded} and~\ref{ass:smooth} together
with the standing boundedness assumption.

\textit{ODE trajectories.}
Boundedness of the actor iterates is supplied directly by
Assumption~\ref{ass:bounded-iterates}; equivalently, the actor update
may be viewed as a projected stochastic approximation on a compact
parameter set. Hence the trajectories of the actor
ODE~\eqref{eq:actor-ode} remain in this compact set.

\subsubsection*{Conclusion}

The following conditions of Borkar's two-timescale Theorem~2
(Chapter~6 of~\cite{borkar2008stochastic}) are satisfied:
\begin{enumerate}[label=(\roman*)]
\item The critic ODE has a unique globally stable equilibrium
  $w^*(\theta)$ for each fixed $\theta$ (Step~2).
\item The map $\theta \mapsto w^*(\theta)$ is continuous (Step~3).
\item The actor mean field $g(\theta, w^*(\theta))
  = \nabla_\theta J_\mathrm{var}(\theta)$ is Lipschitz (Step~4).
\item Both noise sequences satisfy the bounded controlled Markov
  noise conditions with bounded conditional second moments
  (Steps~1, 2, and~4).
\item The step-size separation $\alpha_{\theta,t}/\alpha_{w,t} \to 0$
  holds by Assumption~\ref{ass:stepsizes}.
\item Iterates are bounded by Assumption~\ref{ass:bounded-iterates}.
\end{enumerate}
Therefore $w_t - w^*(\theta_t) \to 0$ almost surely, and the limit
set of $\{\theta_t\}$ is almost surely contained in the set of
stationary points of $J_\mathrm{var}$, i.e., every limit point
$\theta^*$ of $\{\theta_t\}$ satisfies
$\nabla_\theta J_\mathrm{var}(\theta^*) = 0$.
If the stationary points of $J_\mathrm{var}$ are isolated, this limit
set is almost surely a single point, and $\theta_t \to \theta^*$
almost surely. \hfill$\square$

\section{Off-Policy Convergence}
\label{app:proof-offpolicy}
The off-policy actor-critic modifies the on-policy analysis by
replacing stationary averages under $\mu^{\pi_\theta}$ with
behavior-distribution averages under $\mu^b$ and by introducing
importance weights $\rho_t = \pi_\theta(A_t|S_t)/b(A_t|S_t)$.

\begin{assumption}[Off-Policy Coverage and Clipped Importance Weights]
\label{ass:isclip}
The behavior policy $b$ has support over the target policy class:
if $\pi_\theta(a|s) > 0$, then $b(a|s) > 0$.
The importance weights are clipped so that
$0 \leq \rho_t \leq \rho_\mathrm{max} < \infty$ for all $t$.
\end{assumption}

\begin{assumption}[Off-Policy Critic Stability]
\label{ass:offpolicy-critic}
The off-policy projected TD system matrix
$G_b(\theta) = A_b(\theta) - \gamma C_b(\theta)$, computed under the
behavior distribution $\mu^b$, is uniformly positive definite on the
compact parameter set.
\end{assumption}

Under Assumptions~\ref{ass:isclip} and~\ref{ass:offpolicy-critic}, the
actor mean field becomes
\begin{equation}
g_\mathrm{off}(\theta,w)
  = \mathbb{E}_{\mu^b}\!\bigl[\rho(s,a)\hat{Q}_\mathrm{var}(s,a)
    \nabla_\theta\log\pi_\theta(a|s)\bigr],
\end{equation}
where $\hat{Q}_\mathrm{var}(s,a) = Q_w(s,a)
  - \beta\hat{\sigma}_\mathrm{clip}(s,a)$
(following the off-policy policy gradient of Theorem~4 in Jain
et~al.~\cite{jain2021variance}).
The corresponding actor noise is
\begin{equation}
M_{t+1}^\theta
  = \rho_t\,\hat{Q}_\mathrm{var}(S_t,A_t)\,
    \nabla_\theta\log\pi_\theta(A_t|S_t)
    - g_\mathrm{off}(\theta_t, w_t).
\label{eq:actor-noise-off}
\end{equation}
Because the data are generated by the behavior Markov chain, this
noise is centered at its mean field under $\mu^b$ but is not a
literal martingale difference with respect to the full filtration;
it satisfies the controlled Markov noise conditions of Borkar's
framework, as in the on-policy case (Appendix~\ref{app:proof-acconv}).
The clipped ratio $\rho_t \leq \rho_\mathrm{max}$, together with
bounded scores, bounded critic features, bounded iterates, and bounded
$\hat{\sigma}_\mathrm{clip}$, gives the required conditional
second-moment bound:
\begin{equation}
\mathbb{E}\bigl[\|M_{t+1}^\theta\|^2\mid\mathcal{F}_t\bigr]
  \leq \rho_\mathrm{max}^2\,C'(1+\|\theta_t\|^2).
\end{equation}

The critic analysis is analogous to Steps~2 and~3 of
Appendix~\ref{app:proof-acconv}, with all expectations taken under
$\mu^b$ in place of $\mu^{\pi_\theta}$, provided
Assumption~\ref{ass:offpolicy-critic} holds; this positive
definiteness is not automatically implied by coverage or ergodicity
alone and is stated as a separate standing condition.
Under this assumption, the critic tracks the off-policy equilibrium
$w_b^*(\theta)$, and the actor tracks the off-policy ODE
$\dot{\theta} = g_\mathrm{off}(\theta, w_b^*(\theta))$.
Consequently, the limit set of $\{\theta_t\}$ is almost surely
contained in the stationary set of the off-policy actor ODE.
If the stationary points of the induced off-policy objective
$J_{d_0}(\theta)$ are isolated, then $\theta_t$ converges almost
surely to a stationary point $\theta^*$. \hfill$\square$

\begin{remark}
As in the on-policy case, the off-policy proof requires only that
$\hat{\sigma}_\mathrm{clip}$ be bounded, not consistent.
The importance sampling correction $\rho_t$ modifies the actor
mean field and noise but does not affect the contraction of the
penalized Bellman operator, which is established independently
of $\rho_t$ in Appendix~\ref{app:proof-qconv}.
The off-policy critic, unlike its on-policy counterpart, requires
Assumption~\ref{ass:offpolicy-critic} as an additional stability
condition, reflecting the well-known sensitivity of off-policy linear
TD to the choice of behavior policy~\cite{tsitsiklis1997analysis}.
\end{remark}
\section{Full VPAC-BS HTS Evaluation}
\label{app:hts-bs}
 
Table~\ref{tab:bs_full} reports VPAC-BS evaluation metrics across all
$\beta$ values at $\kappa = 1.5$, averaged over 10 seeds $\times$ 40
greedy rollouts.
The best risk-adjusted result ($\beta = 0.1$, shown in \textbf{bold})
achieves 44\% Ret std reduction relative to the within-run baseline.
The irregular pattern at $\beta = 0.01$ (Ret std exceeds the baseline)
is discussed in Section~\ref{subsec:hts-comparison} and attributed to
the mismatch between the bootstrapping estimator's epistemic signal
and the predominantly aleatoric noise structure of the HTS environment.
 
\begin{table}[!t]
\centering
\caption{VPAC-BS evaluation metrics at $\kappa = 1.5$, averaged over
  10 seeds $\times$ 40 rollouts.
  SS~IC~std: steady-state $I_c$ standard deviation (steps 140--200).
  CVaR-10\%: mean of worst 10\% of episode returns (higher is better).
  Best non-baseline entry per column in \textbf{bold}.}
\label{tab:bs_full}
\begin{tabular}{cccccc}
\toprule
$\beta$ & Mean $I_c$ [A] & Ret mean & Ret std & SS IC std [A] & CVaR-10\% \\
\midrule
0.000 (baseline) & 338.3 & 59.06 & 4.48 & -- & 51.27 \\
0.001            & 331.3 & 57.97 & 4.58 & -- & 52.32 \\
0.010            & 340.7 & 59.78 & 5.55 & -- & 50.18 \\
\textbf{0.100}   & 320.5 & 56.74 & \textbf{2.52} & \textbf{9.6} & \textbf{51.76} \\
1.000            & 332.0 & 58.62 & 4.13 & -- & 52.93 \\
\bottomrule
\end{tabular}
\end{table}
 
Figures~\ref{fig:hts_ic_traj_bs} and~\ref{fig:hts_pareto_bs} show
the trajectory and Pareto evaluation for VPAC-BS, analogous to
Figures~\ref{fig:hts_ic_traj} and~\ref{fig:hts_pareto} for VPAC-RS.
Figures~\ref{fig:hts_bc_rs} and~\ref{fig:hts_bc_bs} provide the
rolling IC std and episode return distributions for both methods.
 
\begin{figure}[htbp]
  \centering
  \includegraphics[width=\linewidth]{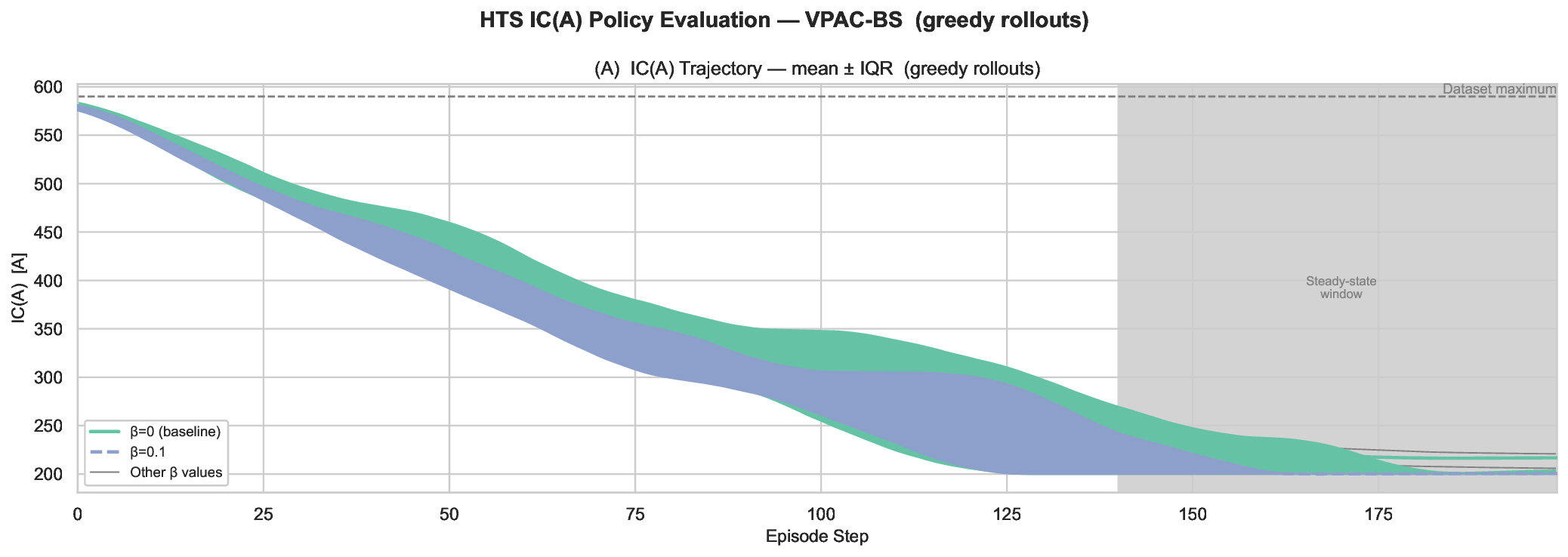}
  \caption{Mean $\pm$ IQR of the $I_c(A)$ trajectory across 400 greedy
    rollouts under the VPAC-BS policy at $\kappa = 1.5$.
    The solid green line is the risk-neutral baseline ($\beta = 0$);
    the dashed blue line is the best risk-adjusted policy ($\beta = 0.1$).
    The $\beta = 0.1$ IQR band is wider than the baseline in the
    transient phase (steps 50--130) before narrowing in the steady-state
    window, indicating that VPAC-BS achieves variance reduction primarily
    in the settled phase of the episode rather than throughout.}
  \label{fig:hts_ic_traj_bs}
\end{figure}
 
\begin{figure}[htbp]
  \centering
  \includegraphics[width=0.9\linewidth]{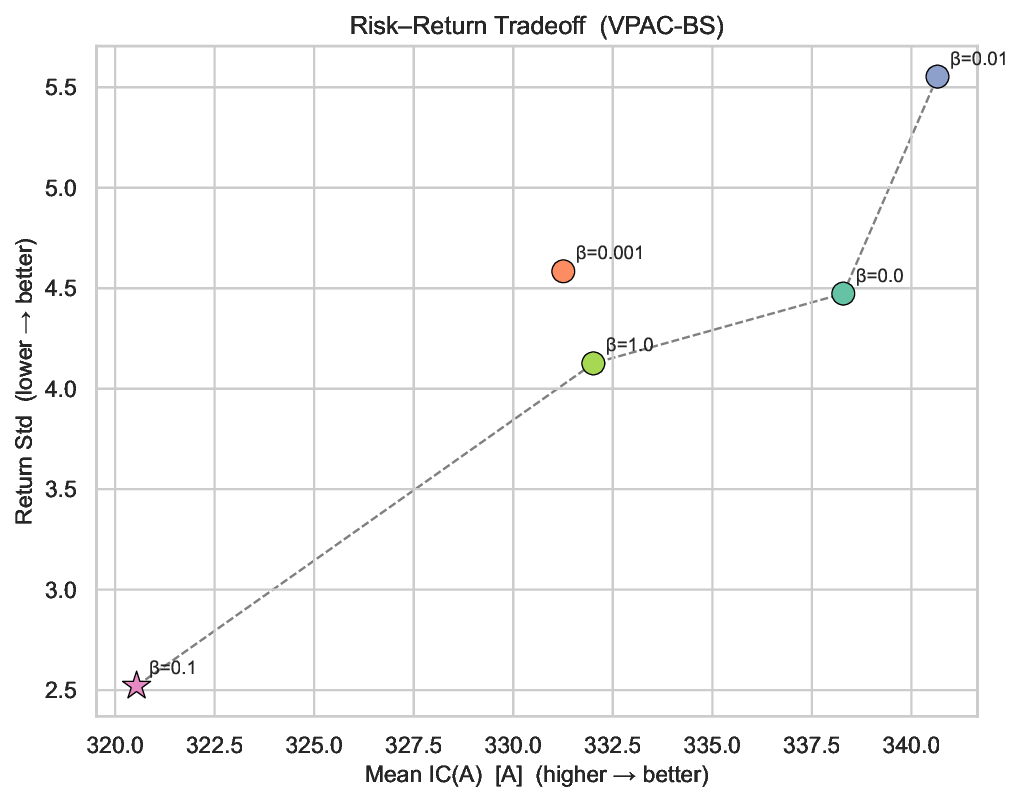}
  \caption{Risk-return Pareto scatter for VPAC-BS at $\kappa = 1.5$.
    Axes and marker conventions follow Figure~\ref{fig:hts_pareto}.
    The starred marker identifies the best risk-adjusted policy
    ($\beta = 0.1$, Ret std $= 2.52$, 44\% below the baseline).
    Four operating points are Pareto-non-dominated; the $\beta = 0.01$
    point is non-dominated on the mean $I_c$ axis (highest mean $I_c$
    of all configurations) despite exhibiting the highest Ret std,
    reflecting the irregular variance signal produced by the
    bootstrapping estimator at small penalty values in this environment.}
  \label{fig:hts_pareto_bs}
\end{figure}
 
\begin{figure}[!t]
  \centering
  \includegraphics[width=\linewidth]{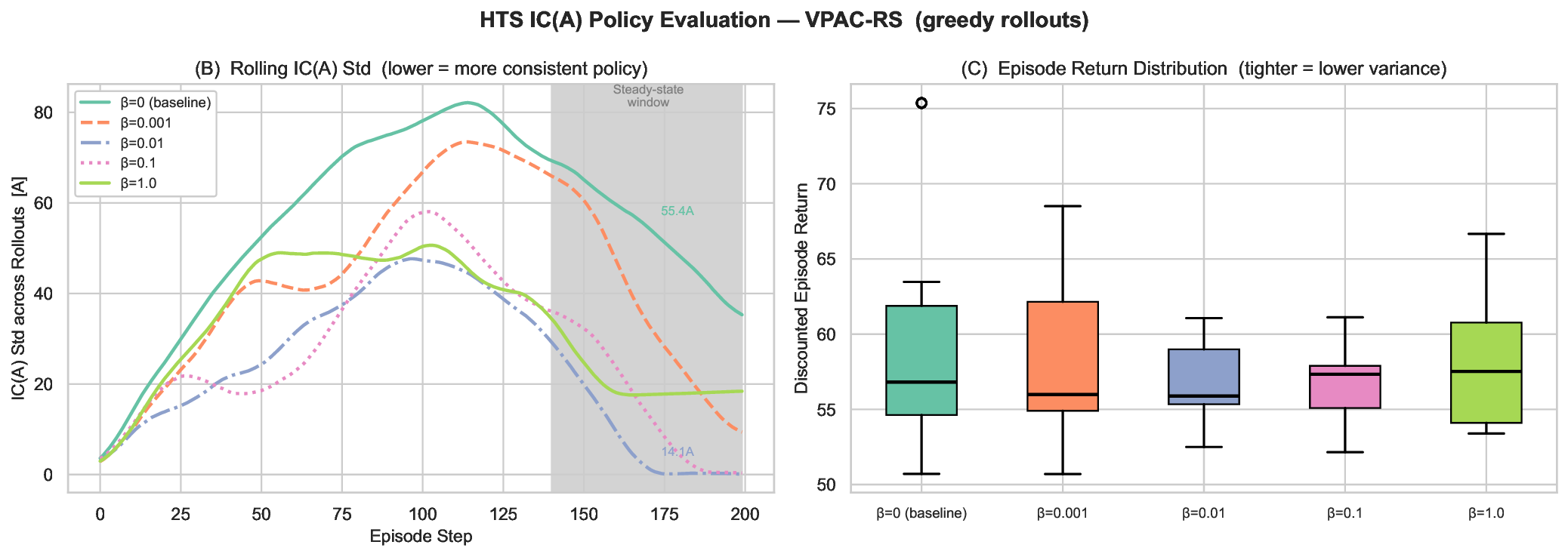}
  \caption{Supporting evaluation panels for VPAC-RS at $\kappa = 1.5$
    (400 greedy rollouts, stochastic environment).
    (B)~Rolling $I_c(A)$ standard deviation across rollouts at each
    episode step, smoothed with a uniform window of 15 steps.
    Annotated values show the mean steady-state IC std (steps 140--200)
    for $\beta = 0$ (55.4~A) and $\beta = 0.01$ (14.1~A).
    The shaded region is the steady-state evaluation window.
    (C)~Episode return distributions across all rollouts; boxes span
    the interquartile range, with the median as the horizontal bar and
    whiskers extending to 1.5 times the IQR.
    The $\beta = 0.01$ and $\beta = 0.1$ policies produce the tightest
    return distributions, consistent with the Ret std values in
    Table~\ref{tab:rs_summary}.}
  \label{fig:hts_bc_rs}
\end{figure}
 
\begin{figure}[!t]
  \centering
  \includegraphics[width=\linewidth]{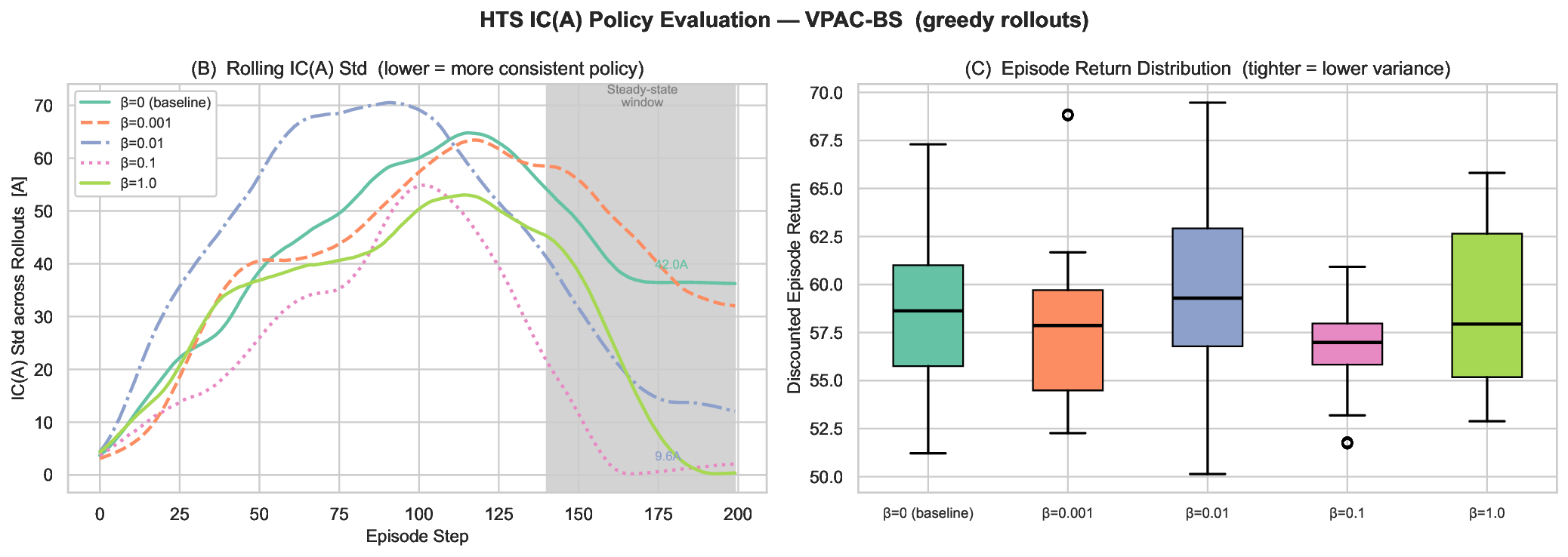}
  \caption{Supporting evaluation panels for VPAC-BS at $\kappa = 1.5$
    (400 greedy rollouts, stochastic environment).
    Panel and axis conventions follow Figure~\ref{fig:hts_bc_rs}.
    (B)~The $\beta = 0.01$ curve substantially exceeds the baseline IC
    std throughout the episode, confirming the pathological result
    reported in Table~\ref{tab:bs_full}.
    The $\beta = 0.1$ policy reaches near-zero IC std in the
    steady-state window (annotated: 9.6~A).
    (C)~Episode return distributions are wider and less ordered across
    $\beta$ values than for VPAC-RS (cf.\ Figure~\ref{fig:hts_bc_rs}C),
    consistent with the weaker and less consistent variance-reduction
    signal produced by the bootstrapping estimator in this environment.}
  \label{fig:hts_bc_bs}
\end{figure}
 
\section{HTS Training Curves}
\label{app:hts-training}
 
Figure~\ref{fig:hts_training} shows the training return and variance
curves for VPAC-RS at $\kappa = 1.5$ across all $\beta$ values.
These curves reflect the evolution of the variance estimator signal
during training, evaluated on stochastic rollouts of the target policy
every 15 episodes for returns and every 30 episodes for variance.
The variance values are computed on normalised discounted returns
(reward divided by 590 at each step) and are therefore on a different
absolute scale from the evaluation Ret std reported in
Table~\ref{tab:rs_summary}.
Despite this scale difference, the relative ordering of $\beta$ values
is qualitatively consistent between training and evaluation: lower
$\beta$ values that achieve smaller Ret std at evaluation also tend to
exhibit lower training variance in the latter half of training.
The y-axis of the variance panel is truncated to highlight relative
ordering; the absolute range is 1.45--1.90 in normalised return units.
 
\begin{figure}[!t]
  \centering
  \includegraphics[width=\linewidth]{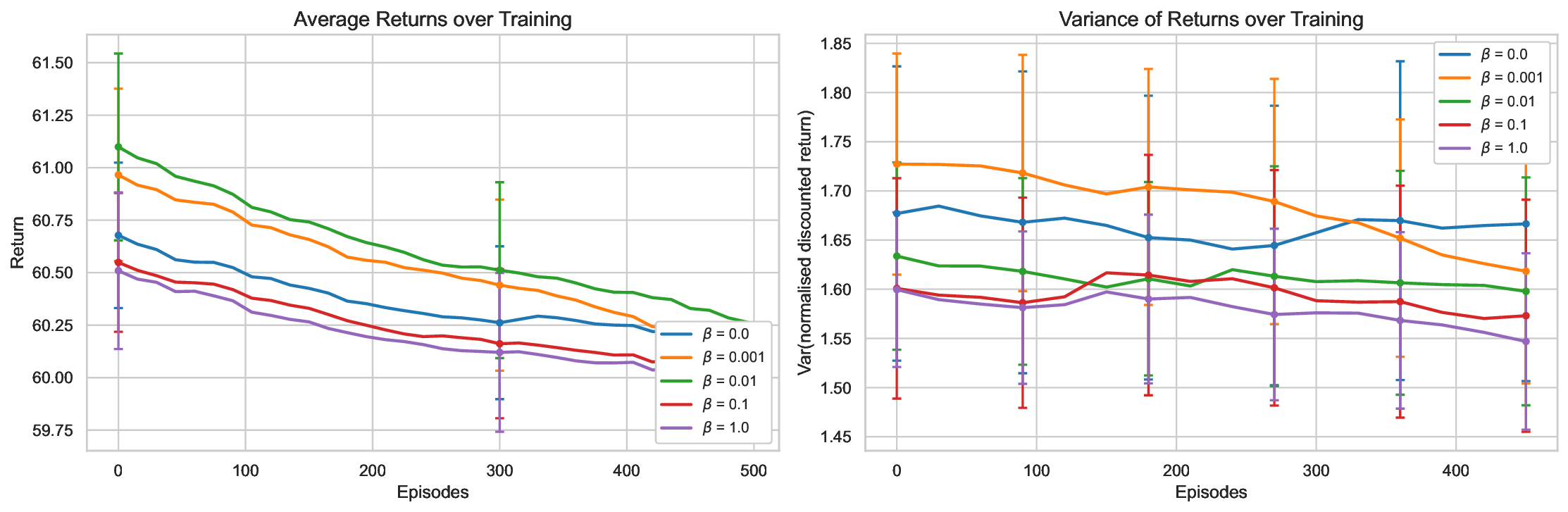}
  \caption{Training curves for VPAC-RS at $\kappa = 1.5$
    (10 seeds; error bars show $\pm 1$ standard error of the mean,
    plotted at sampled intervals for clarity).
    Left: mean episode return over training.
    All $\beta$ values converge to comparable return levels,
    confirming that variance penalisation does not significantly degrade
    mean performance during learning.
    Right: variance of normalised discounted returns estimated during
    training.
    The y-axis is truncated to show relative ordering across $\beta$
    values; absolute variance range is 1.45--1.90.}
  \label{fig:hts_training}
\end{figure}

\begin{figure}[!t]
  \centering
  \includegraphics[width=\linewidth]{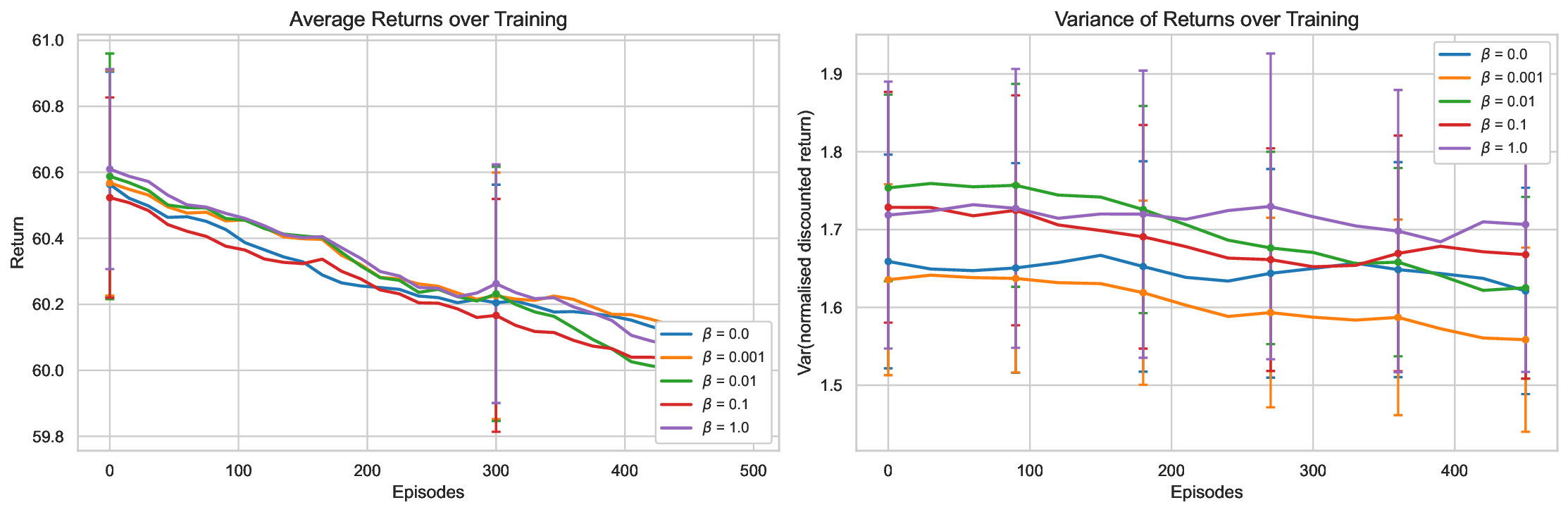}
  \caption{Training curves for VPAC-BS at $\kappa = 1.5$
    (10 seeds; error bars show $\pm 1$ standard error of the mean,
    plotted at sampled intervals for clarity).
    Left: mean episode return over training.
    All $\beta$ values converge to comparable return levels.
    Right: variance of normalised discounted returns estimated
    during training.
    Notably, $\beta = 0.01$ (green) sits above the risk-neutral
    baseline ($\beta = 0$, blue) throughout training, consistent
    with the pathological evaluation result for this configuration
    (Table~\ref{tab:bs_full}), and in contrast to VPAC-RS where
    $\beta = 0.01$ produces the lowest training variance
    (Figure~\ref{fig:hts_training}).
    The y-axis is truncated to show relative ordering; absolute
    variance range is 1.45--1.90.}
  \label{fig:hts_training_bs}
\end{figure}

\end{document}